%% file: camera_ready.tex
\documentclass[nohyperref]{article}

\usepackage{microtype}
\usepackage{graphicx}
\usepackage{url}
\usepackage{algorithm}
\usepackage{algorithmic}
\usepackage{amsmath,amsfonts,bm}
\usepackage{amsthm,amssymb}
\usepackage{bbm}
\usepackage{diagbox}
\usepackage{booktabs}  
\usepackage{color}
\usepackage{xcolor}
\usepackage[colorlinks=true, linkcolor=blue, citecolor=blue]{hyperref}
\usepackage{float}
\usepackage{lipsum}  %
\usepackage{multirow}
\usepackage{graphicx}  %
\usepackage{caption}
\usepackage{subcaption}
\usepackage{newtxmath}
\usepackage{enumitem}
\usepackage{mathtools}
\usepackage{wrapfig}

\input{style/math_commands}

\usepackage[accepted]{style/icml2023}

\newtoggle{arxiv}
\toggletrue{arxiv}

\usepackage[capitalize,noabbrev]{cleveref}  %

\crefname{proposition}{Prop.}{Props.}
\crefname{definition}{Def.}{Defs.}
\crefname{theorem}{Thm.}{Thms.}
\crefname{lemma}{Lemma}{Lemmas}
\crefname{example}{Ex.}{Exs.}
\crefname{equation}{}{}
\makeatletter
\crefname{section}{\S\@gobble}{\S\S}
\crefname{appendix}{\S\@gobble}{\S\S}
\crefname{subsection}{\S\@gobble}{\S\S}
\crefname{subsubsection}{\S\@gobble}{\S\S}
\makeatother
\crefname{figure}{Fig.}{Figs.}
\crefname{wrapfigure}{Fig.}{Figs.}
\crefname{corollary}{Cor.}{Cors.}
\crefname{table}{Table}{Tables}

\theoremstyle{plain}

\newtheorem{innercustomthm}{Theorem}
\newenvironment{customthm}[1]
  {\renewcommand\theinnercustomthm{#1}\innercustomthm}
  {\endinnercustomthm}

\newtheorem{innercustomlemma}{Lemma}
\newenvironment{customlemma}[1]
  {\renewcommand\theinnercustomlemma{#1}\innercustomlemma}
  {\endinnercustomthm}

\newtheorem{innercustomproposition}{Proposition}
\newenvironment{customproposition}[1]
  {\renewcommand\theinnercustomproposition{#1}\innercustomproposition}
  {\endinnercustomthm}

\newcommand{\explainer}[0]{{\color{blue}\textbf{($\star$)} }}

\icmltitlerunning{A Theory of Continuous Generative Flow Networks \hfill \thepage}

\makeatletter

\renewcommand{\paragraph}[1]{\textbf{#1}}
\makeatother

\begin{document}

\twocolumn[
\icmltitle{A Theory of Continuous Generative Flow Networks}

\icmlsetsymbol{equal}{*}

\icmlcorrespondingauthor{Salem Lahlou}{lahlosal@mila.quebec}

\begin{icmlauthorlist}
\icmlauthor{Salem Lahlou}{mila,udem}
\icmlauthor{Tristan Deleu}{mila,udem}
\icmlauthor{Pablo Lemos}{mila,ciela,udem}
\icmlauthor{Dinghuai Zhang}{mila,udem}
\icmlauthor{Alexandra Volokhova}{mila,udem}
\icmlauthor{Alex Hernández-García}{mila,udem}
\icmlauthor{Léna Néhale Ezzine}{mila,udem}
\icmlauthor{Yoshua Bengio}{mila,udem,cifar}
\icmlauthor{Nikolay Malkin}{mila,udem}
\end{icmlauthorlist}

\icmlaffiliation{mila}{Mila}
\icmlaffiliation{ciela}{Ciela Institute}
\icmlaffiliation{udem}{Universit\'e de Montr\'eal}
\icmlaffiliation{cifar}{CIFAR}

\icmlkeywords{}

\vskip 0.3in
]

\printAffiliationsAndNotice{}  %

\begin{abstract}
Generative flow networks (GFlowNets) are amortized variational inference algorithms that are trained to sample from unnormalized target distributions over compositional objects. A key limitation of GFlowNets until this time has been that they are restricted to discrete spaces. We present a theory for generalized GFlowNets, which encompasses both existing discrete GFlowNets and ones with continuous or hybrid state spaces, and perform experiments with two goals in mind. First, we illustrate critical points of the theory and the importance of various assumptions. Second, we empirically demonstrate how observations about discrete GFlowNets transfer to the continuous case and show strong results compared to non-GFlowNet baselines on several previously studied tasks. This work greatly widens the perspectives for the application of GFlowNets in probabilistic inference and various modeling settings.
\vspace*{-4mm}
\end{abstract}

\section{Introduction}
\label{sec:introduction}

Generative flow networks \citep[GFlowNets;][]{bengio2021flow} are an increasingly popular class of methods that amortize sampling from intractable distributions over spaces with a compositional structure by learning a sequential sampling policy. Their applications include the design of biological structures such as molecules~\citep{bengio2021flow, jain2022biological}, Bayesian structure learning~\citep{deleu2022bayesian,nishikawa2022bayesian}, and robust combinatorial optimization \citep{robust-scheduling}. 
Naturally, their development and theoretical foundations~\citep{bengio2021foundations,malkin2022gfnhvi,zimmermann2022variational} have been geared towards environments with discrete structures. 

As many probabilistic inference and modeling problems involve continuous variables, it is natural to ask whether the advantages of GFlowNets, which include stable off-policy learning and the ability to capture many modes of the target distribution, extend to general spaces. For example, molecule design implies specifying relative spatial positions of atoms and benefits from modeling continuous variables, such as torsion angles \citep{torsionaldiffusion2022}, and Bayesian structure learning requires the discovery of not only the structure of the graphical model, but also its parameters. 

As a first attempt at using GFlowNet losses to train an amortized sampler of a unnormalized continuous density, \citet{malkin2022gfnhvi} showed that the off-policy benefits of GFlowNets extend to a toy stochastic control problem. However, a theory justifying the soundness GFlowNet losses in domains with continuous actions has been still lacking. More recently, \citet{li2023cflownets} presented an extension of the flow-matching conditions~\citep{bengio2021flow} to continuous domains; however, this extension relies upon several invalid assumptions, as we expand on in \cref{sec:theory-summary}. %

This paper presents a theory extending all known GFlowNet training objectives to arbitrary spaces. It relies on measurable pointed graphs, a generalization of directed acyclic graphs (DAGs) to measurable spaces, based on Markov kernels. Our main \textbf{theoretical contributions} are an extension of the flow-matching \citep[FM;][]{bengio2021flow}, detailed balance \citep[DB;][]{bengio2021foundations}, and trajectory balance \citep[TB;][]{malkin2022trajectory} conditions and a theorem proving that the learned \emph{forward kernel} samples from the target distribution when any of these conditions is satisfied. These conditions lead to training losses involving density functions and allowing gradient-based learning. Existing losses for discrete GFlowNets are special cases of the ones we state.

Additionally, we provide \textbf{experimental results} in multiple domains with different structures, some of which include action spaces with both discrete and continuous components. These experiments serve both to validate the theoretical claims and to inform practitioners of caveats that are specific to continuous domains. Our comparative experiments confirm that the already-proven advantages of discrete GFlowNets transfer to more general state spaces.

The remainder of the paper is structured as follows:

\cref{sec:background} reviews GFlowNets and work on stochastic sampling;

\cref{sec:theory} presents the theoretical results and a practical summary;

\cref{sec:experiments} is devoted to empirically validating the theory and comparing generalized GFlowNets with baselines.

\section{Background and related work}
\label{sec:background}

\subsection{Stochastic sampling in continuous spaces}

Sequential sampling in continuous spaces has a long history. Specialized Markov chain Monte Carlo (MCMC) methods exist for sampling from continuous or differentiable densities, such as Langevin and Hamiltonian MCMC \cite{langevin,neal-hmc,nuts}.

Another line of work considers stochastic sampling in a finite number of steps. The family of sequential Monte Carlo methods~\citep{doucet01introduction} and the closely related annealed importance sampling \cite{neal-ais} specify a sequence of intermediate target densities with respect to which samplers aim to approximately satisfy detailed balance, but the transition kernel is typically not a learned neural network policy. More recently, learnable-kernel sampling methods, formulated as score-based or stochastic differential equation modeling, have been used for maximum-likelihood generative modeling \citep[e.g.,][]{sohldickstein2015deep,song2019generative,ho2020ddpm,cdld}, as well as for learning to sample from an intractable target density \citep{zhang2022path}. As we show in our experiments, such algorithms can be seen as special cases of GFlowNets where the state space is of a particular form (\cref{sec:pis}, \cref{sec:diffusion}).

Another related direction is stochastic normalizing flows \citep{wu2020stochastic}, which have been interpreted with a Markov chain perspective \citep{hagemann2022stochastic}, relying on Markov kernels and Radon-Nikodym derivatives just as our theory of generalized GFlowNets.

\subsection{Discrete GFlowNets}

GFlowNets were first framed as a reinforcement learning (RL) algorithm~\citep{bengio2021flow}, with discrete state and action spaces, that trains a sampler of a target distribution given by its unnormalized probability mass function, called reward, using a local consistency objective known as flow matching. This contrasts with usual RL algorithms that aim at maximizing a given reward function, but is equivalent to entropy-regularized RL methods in special cases.~\citet{bengio2021foundations} laid out the theoretical foundations of GFlowNets, based on flow networks defined on DAGs, and proposed the detailed balance loss as a more efficient alternative that bypasses the need to sum the flows over large sets of children and parents, opening the door for continuous states. The trajectory balance and subtrajectory balance losses ~\citep{malkin2022trajectory,madan2022learning} have been found to be yet more efficient. \citet{pan2022gafn} proposed a  framework to incorporate intrinsic exploration rewards when training a GFlowNet. \citet{zhang2023unifying,malkin2022gfnhvi,zimmermann2022variational} proved a partial equivalence between GFlowNets and hierarchical variational methods, but showed the superiority of GFlowNets when learning with trajectories sampled \emph{off-policy}. \citet{zhang2022generative} used GFlowNets for approximate maximum-likelihood training of energy-based models, bypassing the need for a given target reward. 

\citet{jain2022biological} used GFlowNets within an active learning loop to design biological sequences. \citet{robust-scheduling} used GFlowNets for an NP-hard combinatorial optimization problem. Other applications include Bayesian structure learning: \citet{deleu2022bayesian,nishikawa2022bayesian} learn posteriors over the combinatorially large space of causal graphs, which are naturally compositional.

\paragraph{Review of discrete GFlowNets.}
Given a non-negative target reward function $R$ on a finite space $\gX$, which coincides with a subset of the vertices of a DAG called the terminating states, GFlowNet training objectives aim at learning transition probabilities $P_F(s' \mid s)$ along the edges of the DAG. The \emph{(forward) action policy} $P_F$ induces a marginal distribution $P_F^\top$ over the terminating states, the final states of trajectories that begin at the designated initial state and sample transitions according to $P_F$. The parameters of the forward policy $P_F$ are sequentially updated using a stochastic gradient of the objective function applied to a trajectory, sampled either from the forward policy itself (\emph{on-policy}), or a modified version thereof in order to incentivize exploration (\emph{off-policy}). When the loss is at a global minimum, the forward policy is able to effectively sample from a probability distribution over $\gX$ proportional to $R$. 

All GFlowNet losses must introduce additional objects into the parametrization to cope with the intractable representation of $P_F^\top(x)$ as a sum of the likelihoods of all (possibly exponentially many) trajectories leading to $x$. For example, the DB and TB objectives use a parametric \emph{backward policy} $P_B(s\mid s')$, which specifies a distribution over the parents of any state in the DAG.

\section{A theory for generalized GFlowNets}
\label{sec:theory}

\subsection{Practical summary}
\label{sec:theory-summary}

\begin{table*}[ht]
\caption{Dictionary between discrete and generalized GFlowNets}
\resizebox{1\linewidth}{!}{
\begin{tabular}{@{}p{0.49\textwidth}p{0.49\textwidth}r@{}}
\toprule
\multicolumn{1}{c}{Discrete GFlowNet} & \multicolumn{1}{c}{Generalized GFlowNet} & Reference\\\midrule
The state space is a finite set with distinguished source and sink states
& The state space is a topological space with distinguished source and sink states, and may consist of both continuous and discrete components 
& \cref{def:measurable_pointed_graph}
\\\midrule
Children and parents of a state $s$
&Supports of the measures $\kappa(s, -), \kappa^b(s, -)$
& \cref{def:measurable_pointed_graph}
\\\midrule
All states are reachable from $s_0$
& All open sets are reachable from $s_0$ with nonzero likelihood 
& \cref{eq:accessibility}
\\\midrule
The state $\bot$ has no outgoing edges 
& The state $\bot$ is absorbing 
& \cref{eq:absorption}
\\\midrule
The state graph is acyclic ($\Rightarrow$ trajectory lengths are bounded) & The measurable pointed graph is finitely absorbing
& \cref{eq:finitely-absorbing}
\\\midrule
State flow $F$, forward policy $P_F$, backward policy $P_B$
& Flow measure $\mu$, forward kernel $P_F$, backward kernel $P_B$
& \cref{def:flow-and-flow-matching,def:db-conditions}
\\\midrule
Transition likelihoods $P_F(-|s)$ only positive along edges
& 
Transition kernels $P_F(s,-)$  absolutely continuous w.r.t.\ $\kappa$
& \cref{def:flow-and-flow-matching}
\\\midrule
Terminating distribution $P_F^\top$
& Terminating state measure $P_\top$
& \cref{def:trajectory-measures,thm:fm-implies-correct-sampling}
\\\midrule
Flow-matching implies sampler matches reward function $R$
& 
Flow-matching implies sampler matches reward measure $R$
& \cref{def:flow-and-flow-matching,thm:fm-implies-correct-sampling}
\\\midrule
Detailed balance: $F(s)P_F(s'|s)=F(s')P_B(s|s')$
&
Detailed balance: $\mu(ds)P_F(s,ds')=\mu(ds')P_B(s',ds)$
& \cref{def:db-conditions}
\\\midrule
Trajectory balance: $ZP_F(\tau)=R(x_\tau)P_B(\tau \mid x_\tau)$
&
Trajectory balance: $Z P_F(s_0, ds_1) \dots P_F(s_n, \{\bot\}) = R(d s_n) P_B(s_n, ds_{n-1}) \dots P_B(s_1, \{s_0\})$
& \cref{def:tb-conditions}
\\\bottomrule
\end{tabular}
}
\label{tab:dictionary}
\end{table*}

A summary of the key differences and analogies between discrete and generalized GFlowNets is provided in \cref{tab:dictionary}, and the precise way in which discrete GFlowNets are special cases of generalized GFlowNets is stated in \cref{appendix:pointed-graphs}. 

The theory we develop in this section shows that the main losses used to train discrete GFlowNets, namely the detailed balance~\citep{bengio2021foundations} and the trajectory balance~\citep{malkin2022trajectory} losses, {\em naturally} extend to generalized GFlowNets, simply by replacing probability mass functions with probability density functions. Most importantly, however, the soundness of the theory relies upon the following assumptions, which need to be carefully verified when training a GFlowNet in an infinite space:
\begin{enumerate}[label=(\arabic*),nosep,left=0pt]
    \item The structure of the state space must allow all states to be reachable from the source state $s_0$ \cref{eq:accessibility};
    \item The structure must ensure that the number of steps required to reach any state from $s_0$ is bounded \cref{eq:finitely-absorbing};
    \item The learned probability measures need to be expressed through densities over states, rather than over actions.
\end{enumerate}
These assumptions are naturally verified in discrete domains, as long as the state space is described by a pointed directed acyclic graph~\citep{bengio2021foundations}.

\paragraph{On previous attempts to train continuous GFlowNets.} 
While \citet{li2023cflownets} proposed to train continuous GFlowNets by writing the flow matching conditions as integrals rather than sums, assumptions (1) and (3) are violated in a critical way. First, the environments considered in \citet{li2023cflownets}'s experiments violate assumption (1), without which the main GFlowNet training theorems do not hold. 
Second, regarding (3), \citet{li2023cflownets} implicitly assumes that for a state $s$ and flow function $F(s\to s)$,
\[\int_{s': s\rightarrow s'}F(s \to s')\,ds'=\int_aF(s\to T(s,a))\,da,\]
where the second integral is taken over actions and $T(s,a)$ is the state reached by taking action $a$ from $s$. This change of variables is invalid in general: 
\textit{the integrand on the right side is missing the Jacobian term $\frac{dT(s,a)}{da}$},
which need not equal $1$. In particular, it does not equal $1$ in the environments studied by \citet{li2023cflownets} (although it may hold in special cases, such as sampling in Euclidean spaces where $T(s,a)=s+a$). These issues are concerning for the scope of that method's applicability.

\subsection{Structured state space}
\label{sec:structured-state-space}

\paragraph{Note.} To help the reader form a mental picture, we list the concepts introduced and their discrete analogues in \cref{tab:dictionary} and formally state the connection in \cref{appendix:pointed-graphs}. Paragraphs marked \explainer explain the meaning of the technical results.

\paragraph{\explainer How could one describe a structure in general spaces, similar to DAGs on finite sets? } In finite sets, it would suffice to enumerate the child sets and parent sets of all states, with the constraint that $s'$ is a child of $s$ if and only if $s$ is a parent of $s'$. In general state spaces, however, enumeration is replaced by measure. One could thus define, for each state, a measure on the state space describing what states can be accessed in one step.

The structured state spaces we consider will be called {\em measurable pointed graphs} and rely on {\em transition kernels}~\citep{nummelin2004general,cappe2009inferencehmm, petritis2015markov}, of which we recall the definition in~\cref{appendix:transition-kernels}.

\begin{definition}[Measurable pointed graph]
A \emph{measurable pointed graph} $G=(\bar{\gS}, \gT, \Sigma, s_0, \bot, \kappa, \kappa^b, \nu)$ consists of:
\begin{itemize}[nosep,left=0pt]
\item A topological space $(\bar{\gS}, \gT)$, where $\gT$ is the set of open subsets of $\bar{\gS}$ and $\Sigma$ is the Borel $\sigma$-algebra associated to the topology on $\bar{\gS}$;
\item A pair of distinct distinguished states $s_0\in\bar\gS$ and $\bot\in\bar\gS$, called the \emph{source state} and \emph{sink state}, such that $\{s_0\}$ and $\{\bot\}$ are both open and closed sets. We define $\gS=\bar{\gS}\setminus\{\bot\}$ and $\gS^\circ=\gS\setminus\{s_0\}$, so the topology on $\bar{\gS}$ is the disjoint union topology on $\{s_0\}$, $\{\bot\}$, and $\gS^\circ$.
\item A $\sigma$-finite transition kernel $\kappa$ on $(\bar{\gS}, \Sigma)$, called the \emph{reference kernel},
\item A $\sigma$-finite transition kernel $\kappa^b$ on $(\bar{\gS}, \Sigma)$, called the \emph{backward reference kernel},
\item A \emph{strictly positive} $\sigma$-finite measure $\nu$ on $(\bar{\gS}, \Sigma)$, called the \emph{reference measure},
\end{itemize}
such that the following conditions hold:
\begin{align}
    &\forall B \in \gT \setminus \{ \emptyset \}\quad\exists n \geq 0 : \kappa^n(s_0, B) > 0,\label{eq:accessibility} \\
    &\kappa(\bot,-) = \delta_\bot, \label{eq:absorption} \\
    &\forall B \in \Sigma, \ s\mapsto \kappa(s, B) \text{ is continuous,}\\
    &\forall B \in \Sigma \otimes \Sigma, \ ((s_0, s_0) \notin B, \ (\bot, \bot) \notin B) \Rightarrow \nonumber\\
    &\quad\quad\quad\quad\quad\quad\quad\quad\nu \otimes \kappa(B) = \nu \otimes \kappa^b(B) ,\label{eq:backward-kernel}\\
    &\forall B \in \Sigma, \ \kappa^b(s_0, B) = 0 ,\label{eq:trivial-measure}\\
    &\forall s \in \gS, \ \kappa(s, \{\bot\}) > 0 \Rightarrow \kappa(s, \{\bot\}) = 1. \label{eq:kappa-terminal-1}
\end{align}
The measurable pointed graph is called \emph{finitely absorbing} if
\begin{equation}
\label{eq:finitely-absorbing}
\exists N>0: {\rm supp}(\kappa^N(s_0,-))=\{\bot\},
\end{equation}
in which case the minimal such $N$ is called the maximal trajectory length.
\label{def:measurable_pointed_graph}
\end{definition}

\explainer The reference transition kernel $\kappa$ provides a notion of ``structure'' of the state space. The support of $\kappa(s, -)$ (resp.\ $\kappa^b(s, -)$) can be thought of as the child set (resp.\ parent set) of the state $s$. For example, in a discrete graph, $\kappa(s,-)$ could be uniform over the children of $s$. The reference kernel is not a policy to be sampled, but an object needed to define probability densities of policies. The measure $\nu$, the reference with respect to which flows and rewards are defined, is typically a simple measure, such as the counting measure on a discrete set or the standard Lebesgue measure on a Euclidean space.

In practice, if the structure is only defined by the reference kernel $\kappa$, then $\nu$ and $\kappa^b$ satisfying the conditions of \cref{def:measurable_pointed_graph} can be defined from $\kappa$ under some mild assumptions, as we discuss in \cref{prop:existence-of-backward-kernel} in \cref{appendix:backward-kernels}.

From now on, we fix a finitely absorbing measurable pointed graph $G=(\bar{\gS}, \gT, \Sigma, s_0, \bot, \kappa, \kappa^b, \nu)$ with maximal trajectory length $N$.

\begin{definition}[Terminating states]
   The set of \emph{terminating states} $\gX$ is defined by:
    \begin{equation}
        \gX = \{ s \in \gS \ : \ \kappa(s, \{\bot\}) > 0 \}.
        \label{eq:terminating-states}
    \end{equation}
    \vspace{-1.5em}
    \label{def:terminating-states}
\end{definition}

\explainer Terminating states are ones from which one can transition to $\bot$ with positive probability. Any transition kernel can be sampled for $n$ steps, yielding a measure over $n$-step trajectories and a marginal measure over states reached after $n$ steps. As described in the Appendix (\cref{sec:terminating-state-measure}), this can be used to define the marginal terminating measure $P_\top$ of a transition kernel $P_F$, used in the next section.

\subsection{Flows}
\begin{definition}[Flows and flow-matching conditions]
    \label{def:flow-and-flow-matching}
    Given a $\sigma$-finite measure $\mu$ over $(\bar{\gS}, \Sigma)$ that is absolutely continuous w.r.t.\ $\nu$ (we write $\mu \ll \nu$), and a $\sigma$-finite Markov kernel $P_F$ on $(\bar{\gS}, \Sigma)$ (i.e. a transition kernel such that each $P_F(s, -)$ is a probability measure) satisfying:
    \begin{enumerate}[label=(\arabic*),left=0pt,nosep,]
    \item $P_F(s,-) \ll \kappa(s,-)$ for every $s \in \bar\gS$,
    \item $s \mapsto P_F(s, B)$ is continuous for every $B \in \Sigma$,
    \end{enumerate}
    $P_F$ is said to be a \textbf{forward kernel} over $G$. We say that the tuple $F=(\mu, P_F)$ satisfies the \textbf{flow-matching (FM) conditions} if for any bounded measurable function $f: \bar{\gS} \rightarrow \sR$ satisfying $f(s_0)=0$, we have
        \begin{equation}
            \int_{\bar{\gS}} f(s')\mu(ds') = \iint_{\gS \times \bar{\gS}} f(s')\mu(ds)P_{F}(s, ds').
            \label{eq:flow-matching-conditions}
        \end{equation}
    In which case, we say that $F$ is a \textbf{flow} over $G$.
\end{definition}

\explainer The condition of absolute continuity w.r.t.\ the reference kernel $\kappa$ indicates that the flow $F$ must follow the ``structure'' of the measurable pointed graph, by assigning positive measure only to parts of the space where the measure induced by $\kappa$ is also positive. The kernel $P_F$ can be represented through a density function with respect to $\kappa$, which represents a probability \emph{mass} (if the action space is discrete) or a probability \emph{density} (if it is continuous). This  allows to write conditions such as \cref{eq:flow-matching-conditions} using densities (Radon-Nikodym derivatives), thus providing practical loss functions to train GFlowNets. We expand on this point in \cref{sec:training-losses}.

\begin{definition}[Reward-matching conditions]
\label{def:reward-matching-conditions}
Let $F=(\mu, P_F)$ be a flow over $G$. Given a positive and finite measure $R$ over $\gX$, called the reward measure, 
 satisfying $R \ll \nu$, the flow $F$ is said to satisfy the \emph{reward-matching condition} w.r.t.\ $R$ if we have:%
    \vspace*{-1mm}\begin{equation}\vspace*{-1mm}
        R(dx)=\mu(dx)P_F(x,\{\bot\}).
        \label{eq:boundary-condition}
    \end{equation}
\end{definition}

The following theorem, proved in \cref{appendix:proofs}, ascertains that, similar to discrete GFlowNets, when the flow and reward matching conditions are satisfied, then recursively sampling from the Markov kernel $P_F$ starting from $s_0$ (until reaching $\bot$) yields samples from the normalized reward.
\begin{theorem}
    If $F=(\mu, P_F)$ is a flow over $G$, that satisfies the reward matching conditions \cref{eq:boundary-condition} w.r.t.\ a measure $R$, then the corresponding terminating state measure $P_\top$ (\cref{def:terminating-states-measures}) is a probability measure and satisfies for all $B \in \Sigma_{| \gX}$:
    \vspace*{-1mm}\begin{equation}\vspace*{-1mm}
        \label{eq:terminating-state-distribution}
        P_\top(B) = \frac{1}{R(\gX)} R(B).
    \end{equation}
    \vspace*{-1em}
    \label{thm:fm-implies-correct-sampling}
\end{theorem}

\explainer $R(\gX)$, the reward measure taken over the set of all terminating states $\gX$, corresponds to the total reward or partition function $Z$ of GFlowNets. Certain conditions  (Def. \ref{def:flow-and-flow-matching} and \ref{def:reward-matching-conditions}) on $\mu$, which represents a state flow, and $P_F$, which represents a policy, imply that the marginal terminating distribution of the policy is proportional to the reward. These conditions correspond to the ``flow in = flow out'' condition at vertices of a discrete DAG.

\subsection{Detailed balance and trajectory balance}
\label{sec:detailed-balance-trajectory-balance}
In finite GFlowNets, the detailed balance conditions~\citep{bengio2021foundations} and the trajectory balance conditions~\citep{malkin2022trajectory} were converted into training objectives in order to sample from a target unnormalized distribution. In this section, we present analogous conditions for general measurable pointed graphs.
\begin{definition}
    Let $\mu$ be a $\sigma$-finite measure over $(\bar{\gS}, \Sigma)$ such that $\mu \ll \nu$, $P_F$ a forward kernel over $G$, and $P_B$ a transition kernel on $(\bar{\gS}, \Sigma)$ such that:
    \begin{enumerate}[label=(\arabic*),left=0pt,nosep,]
        \item $P_B(s, -) \ll \kappa^b(s, -)$ for every $s \in \bar{\gS}$,
        \item $s \mapsto P_B(s, B)$ is continuous for every $B \in \Sigma$,
        \item $P_B(s, -)$ is a probability measure for every $s \neq s_0$,
    \end{enumerate}
     $P_B$ is then said to be a \textbf{backward kernel} over $G$. We say that $(\mu, P_{F}, P_{B})$ satisfy the \textbf{detailed balance (DB) conditions} if for any bounded measurable function $f: \gS \times \bar{\gS} \rightarrow \sR$ satisfying $f(s, s_0) =0$ for every $s \in \gS$, we have
    \begin{align}
        &\iint_{\gS \times \bar{\gS}} f(s, s') \mu(ds)P_{F}(s, ds') \label{eq:db-conditions}\\
        &\qquad \qquad = \iint_{\gS \times \bar{\gS}} f(s, s')\mu(ds')P_{B}(s', ds).\nonumber
    \end{align}
    \label{def:db-conditions}
        \vspace{-1em}
\end{definition}
The following proposition, proved in \cref{appendix:proofs}, shows an equivalence between the DB and FM conditions.
\begin{proposition}
    If $(\mu, P_{F}, P_{B})$ satisfy the detailed balance conditions in \cref{def:db-conditions}, then $F = (\mu, P_{F})$ satisfies the flow-matching conditions in \cref{def:flow-and-flow-matching} and is thus a flow.
    \label{prop:DB-implies-FM}
\end{proposition}

\begin{definition}
Let $P_F$ be a forward kernel over $G$,  $P_B$ a backward  kernel over $G$, and $Z \in \mathbb{R}_+$. Let $R$ be a positive finite measure on $\gX$.  The triple $(Z, P_F, P_B)$ satisfies the \textbf{trajectory balance (TB) conditions} w.r.t.\ $R$ if for any $n\geq0$ and any bounded measurable function $f: \bar{\gS}^{n+2} \rightarrow \sR$:
   \vspace*{-1mm} \begin{align}\vspace*{-1mm}
        &\hspace{-4mm}\int_{\bar{\gS}^{n+2}} Zf(s,\overrightarrow{s_{1:n+1}}) \indicator_{s_n \neq \bot, s_{n+1} =\bot} P_F^{\otimes{n+1}}(s_0, ds\,\overrightarrow{ds_{1:n+1}})\label{eq:tb-conditions} \\
        &= \int_{\bar{\gS}^{n+1}} \indicator_{s=s_0}f(s,\overrightarrow{s_{1:n}},\bot)R(ds_n)P_{B}^{\otimes{n}}(s_n, ds'\,\overrightarrow{ds_{n-1:1}}\,ds),\nonumber
    \end{align}
    where $\overrightarrow{s_{1:n}}$ denotes $(s_1, \dots, s_n)$ and $\overrightarrow{ds_{1:n}}$ denotes $ds_1\dots ds_n$.
    \label{def:tb-conditions}
\end{definition}

The following proposition, proved in \cref{appendix:proofs}, shows an equivalence between the TB and both the FM and reward matching conditions.
\begin{proposition}
    If $(Z, P_F, P_B)$ satisfy the TB conditions \cref{eq:tb-conditions} w.r.t.\ a measure $R$, then $F=(\mu, P_B)$, where $\mu$ is defined by:
    \begin{enumerate}[label=(\arabic*),left=0pt,nosep,]
    \item $\mu(\{\bot\})=\mu(\{s_0\}) = Z$
    \item $\forall B \in \Sigma_{| \gS}$:
    $\mu(B) = \mu(\{s_0\})\sum_{n=0}^{\infty} P_F^n(s_0, B) $
    \end{enumerate}
    satisfies both the flow-matching conditions \cref{eq:flow-matching-conditions} and the reward matching conditions \cref{eq:boundary-condition} w.r.t.\ $R$.
    \label{prop:TB-implies-FM-RM}
\end{proposition}

\explainer Analogues of the DB and TB conditions for discrete GFlowNets were stated and shown to imply the FM conditions. In the next section, they will be used to construct training objectives for parametric policies.

\subsection{Training losses for GFlowNets}
\label{sec:training-losses}
Above, we have presented three conditions under which a sampler based on a Markov kernel $P_F$ samples from the normalized version of a given reward measure. In practice, similar to discrete GFlowNets, the objects of interest ($\mu$, $P_F$, $P_B$, $Z$) are parametrized by a vector $\theta$, and the goal is to learn $\theta$ using gradient-based learning. In this section, we derive losses corresponding to the previous objectives.

We recall the Radon-Nikodym theorem that states that for any two given $\sigma$-finite measures $p$ and $q$ on a measurable space $(U, \gU)$ satisfying $p \ll q$, there exists a measurable function $f: U \ra \R_+$, which is unique up to a set of measure zero under $q$, called the density or the Radon-Nikodym derivative of $p$ w.r.t.\ $q$, such that:
\vspace{-2mm}\begin{align}\vspace{-2mm}
    \forall A \in \gU, \ p(A) = \int_A f(u) q(du).
\end{align}

This theorem is convenient as it allows to bypass the need to define the measures $\mu$, $P_F(s, -)$, $P_B(s, -)$ on every measurable set, and only requires parametrizing the corresponding densities (w.r.t.\ $\nu$, $\kappa(s, -)$, and $\kappa^b(s, -)$ respectively).

\begin{definition}[Losses]
    Let $u: \gS \ra \R_+$, $p_{F}: \gS \times \bar{\gS} \ra \sR_{+}$, and $p_{B}: \gS \times \gS \ra \sR_{+}$ be three functions, and $Z \in \R_+$ a scalar, all parametrized by a vector $\theta$, and satisfying for every $\theta$:
    \vspace{-2mm}\begin{align}\vspace{-2mm}
        &\forall s \in \bar{\gS}, \ \int_{\bar{\gS}} p_{F}(s, s'; \theta) \kappa(s, ds') = 1 \\ 
        &\forall s' \in \bar{\gS}, \ \int_{\bar{\gS}} p_{B}(s', s; \theta) \kappa^b(s', ds) = 1
    \end{align}
    The \textbf{flow-matching} (FM) loss is defined for every $s' \in \gS$ as:
    \vspace{-2mm}\begin{align*}\vspace{-2mm}
        L_{FM}(s'; \theta) = \left( \log \frac{\int_{\gS} u(s; \theta) p_{F}(s, s'; \theta) \kappa^b(s', ds)}{u(s'; \theta)  } \right)^2
    \end{align*}
    The \textbf{detailed balance} (DB) loss is defined for every $(s, s') \in \gS \times \gS$ as:
    \vspace{-2mm}\begin{align*}\vspace{-2mm}
        L_{DB}(s, s'; \theta) = \left(\log \frac{u(s; \theta) p_{F}(s, s'; \theta)}{u(s'; \theta) p_{B}(s', s; \theta)} \right)^2
    \end{align*}
    Denoting by $r$ the density of the reward measure $R$ w.r.t.\ the reference measure $\nu$, the \textbf{reward-matching} (RM) loss is defined for any $x \in \gX$ as:
    \vspace{-2mm}\begin{align*}\vspace{-2mm}
        L_{RM}(x;\theta) = \left(\log \frac{u(x;\theta)p_{F}(x, \bot; \theta)}{r(x)}\right)^{2}
    \end{align*}
    Finally, the \textbf{trajectory balance} (TB) loss is defined for every complete trajectory $\tau=(s_0, s_1, \dots, s_n, s_{n+1}) \in  \{s_0\} \times \gS^{n} \times \{\bot\}$ (also denoted $\overrightarrow{s_{0:n+1}}$) where $s_n \in \gX$ and $s_{n+1} = \bot$ as:
    \vspace{-2mm}\begin{align*}\vspace{-2mm}
        L_{TB}^n(\tau; \theta) = \left(\log \frac{Z(\theta) \prod_{t=0}^n p_{F}(s_{t}, s_{t+1}; \theta)}{r(s_n) \prod_{t=0}^{n-1} p_{B}(s_{t+1}, s_{t}; \theta)}  \right)^2.
    \end{align*}
\end{definition}
Note that one could derive in a similar fashion a subtrajectory balance loss, similar to the one used in discrete GFlowNets~\citep{madan2022learning}.

\explainer The above losses resemble discrete GFlowNet losses. When the action space is discrete, and the reference measures are the counting measures over vertices of a DAG, $p_F(s,s')$ is a transition probability $P_F(s'|s)$. When it is continuous, it represents a conditional probability density over $s'$, given $s$.

Conversely, from functions $u(-;\theta), p_F(-;\theta), p_B(-;\theta)$, we can define a measure $\mu(-;\theta)$ on $(\bar{\gS}, \Sigma)$ whose density w.r.t.\  $\nu$ is $u$ and forward and backward kernels $P_F(-; \theta), P_B(-;\theta)$ such that $p_F(s, -; \theta)$ and $p_B(s', -; \theta)$ are their densities of w.r.t.\ $\kappa(s, -)$ and $\kappa^b(s, -)$, respectively\footnote{The measures at $\bot$ are irrelevant.}.

\explainer The following theorem, proved in \cref{appendix:proofs}, ensures that, similar to the discrete case, minimizing the losses above leads to samplers of the right probability measure.
\begin{theorem}
    \begin{enumerate}[left=0pt,nosep,label=(\arabic*),wide]
    \item If $L_{FM}(-;\theta) = 0$ $\nu$-almost surely, then $F=(\mu, P_{F})$ is a flow (i.e. satisfies the flow-matching conditions in \cref{def:flow-and-flow-matching}).
    \item If $L_{DB}(-;\theta) = 0$ $\nu \otimes \kappa$-almost surely, then $(\mu, P_{F}, P_{B})$ satisfy the detailed balance conditions in \cref{def:db-conditions}.
    \item If $L_{RM}(-;\theta)=0 $ $\nu_{|\gX}$-almost surely, then $(\mu, P_F)$ satisfies the reward matching conditions in \cref{eq:boundary-condition}.
    \item If $L_{TB}^n(-; \theta)=0$ $((\nu \otimes \kappa^{\otimes n+1})_{|\{s_0\}\times \gS^n \times \{\bot\}}$)-almost surely for every $n \geq 0$, then $(Z\nu(\{s_0\}), P_F, P_B)$ satisfy the trajectory balance condition in \cref{def:tb-conditions}.
    \end{enumerate}
    \label{thm:zero-losses-implies-conditions}
\end{theorem}

An important consequence of \cref{thm:zero-losses-implies-conditions} is that if we can find density functions that achieve zero loss using any of the above objectives almost surely, in addition to the reward-matching loss, then we obtain a way to sample terminating states (i.e., elements of $\gX$) proportionally to the reward measure $R$, according to \cref{thm:fm-implies-correct-sampling}.

\paragraph{Training generalized GFlowNets.} 
The FM, DB, and TB losses can be minimized using states (resp.\ pairs of subsequent states, trajectories) obtained from trajectories sampled from a training policy $\pi$, which can be $P_F$ itself (\emph{on-policy}), or a modification of it to encourage exploration (\emph{off-policy}). Thm.~\ref{thm:zero-losses-implies-conditions} suggests that the parameters $\theta$ could be updated with stochastic gradients $\E_{\tau=\overrightarrow{s_{0:n+1}} \sim \pi}[\nabla_\theta\gL]$, where $\gL$ is  $\sum_{t=1}^n   L_{FM}(s_t; \theta) + \alpha  L_{RM}(s_n; \theta)$, or $\sum_{t=0}^n  L_{DB}(s_t, s_{t+1}; \theta) + \alpha  L_{RM}(s_n; \theta)$ or $ L_{TB}(\tau; \theta)$.

\section{Experiments}
\label{sec:experiments}

\begin{figure}[t]
    \centering
        \raisebox{1.3mm}{\includegraphics[width=0.36\linewidth]{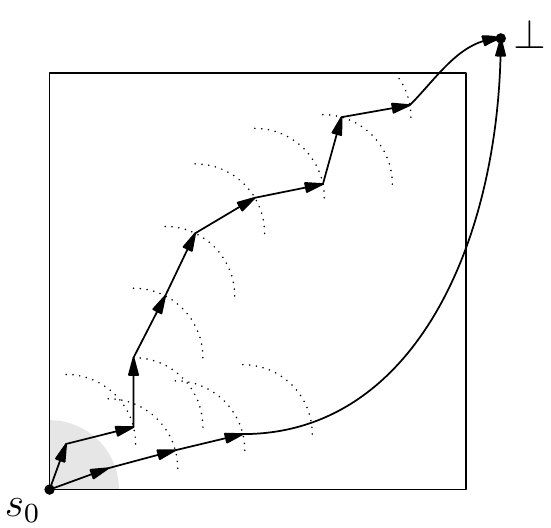}}
        \includegraphics[width=0.63\linewidth]{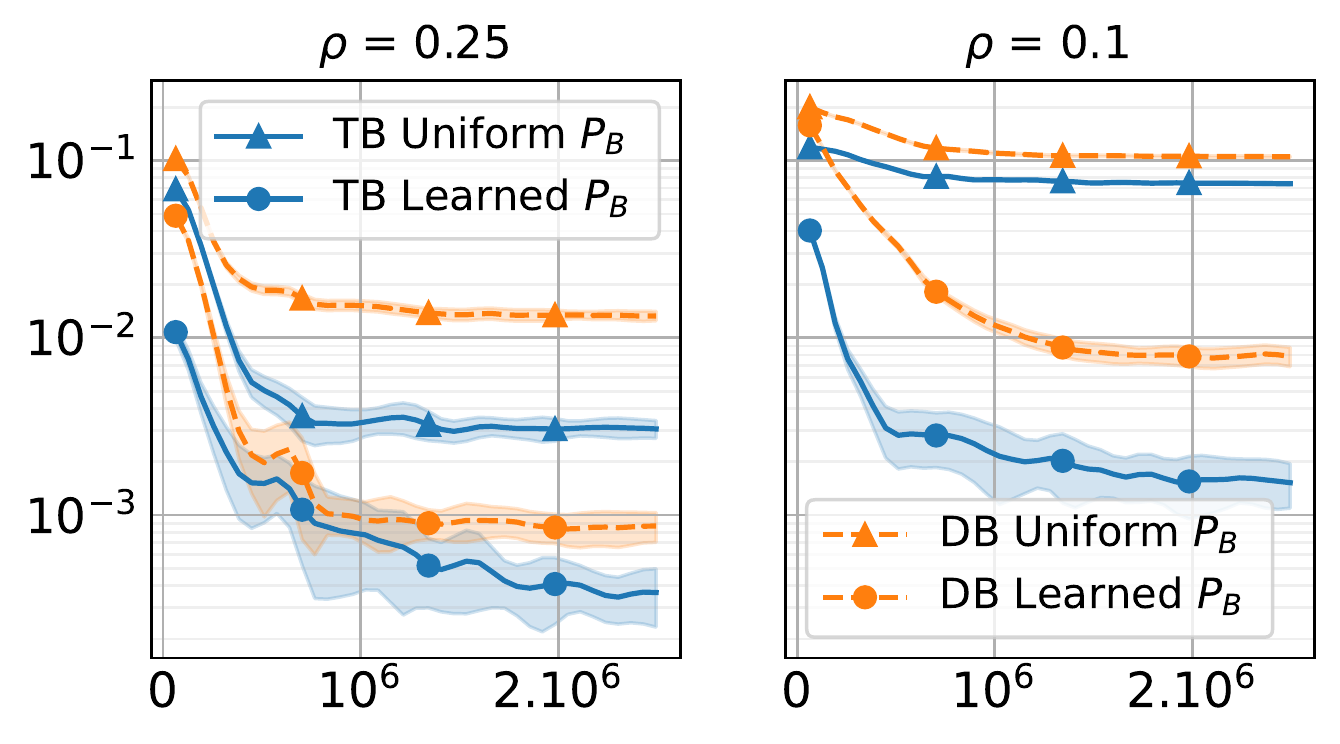}\\\vspace{-1em}
    \caption{\textbf{(a)} Measurable pointed graph structure of the environment in \cref{sec:continuous-grid}: starting at $s_0$, the first action makes a step within the grey quarter-disc, and subsequent actions make steps of a fixed size or terminate. \textbf{(b)} Evolution of the JSD during training of TB and DB, with both a uniform $P_B$ and a learned $P_B$, for $\rho=0.25$; \textbf{(c)} $\rho=0.1$. x-axis is the number of sampled trajectories. Shaded areas represent standard deviations across 6 runs.}
    \vspace*{-4mm}
    \label{fig:grid-structure-jsd}
\end{figure}

\subsection{A synthetic continuous environment}
\label{sec:continuous-grid}

In this section, we study a synthetic environment inspired by the hypergrid environment~\citep{bengio2021flow,malkin2022trajectory,malkin2022gfnhvi}, with varying trajectory lengths and a pointed graph structure imposing a mixed discrete and continuous probability measure for the policy $P_F$. Code for these experiments can be found at \url{https://github.com/saleml/continuous-gfn}.

\paragraph{Structure of the state space.} The measurable pointed graph is specified by $\gS = [0, 1]^2$, and $s_0 = (0, 0)$. A hyperparameter $\rho$, called the step size, controls the maximal trajectory length. $\kappa(s_0, -)$ is the Lebesgue measure on $D_0$, the northeastern quarter disk of radius $\rho$ centered at $s_0$. When $s \neq s_0$, and $||s||<1 - \rho$, $\kappa(s, -)$ is the sum of the one-dimensional Lebesgue arclength measure on $C_s^+$ (the intersection of the northeastern quarter circle of radius $\rho$ centered at $s$ and $\gS$)
and the Dirac measure $\delta_{\bot}$. Finally, when $||s||>1 - \rho$, $\kappa(s, -) = \delta_{\bot}$. The forward structure is depicted in \cref{fig:grid-structure-jsd}(a).
The backward reference kernel $\kappa^b$ is defined similarly.

The reference measure $\nu$ is the sum of the Lebesgue measure on $\gS$, $\delta_{s_0}$, and $\delta_{\bot}$. All states besides $s_0$ are terminating. 

The reward measure $R$ on $\gX$ is specified by a density function $r$  depicted in \cref{fig:reward_kde_hypergrid}(a). The densities $p_F$ and $p_B$ are parametrized with mixtures of Beta distributions for the continuous components. 

\begin{figure}[t]
    \centering
    \includegraphics[width=\linewidth]{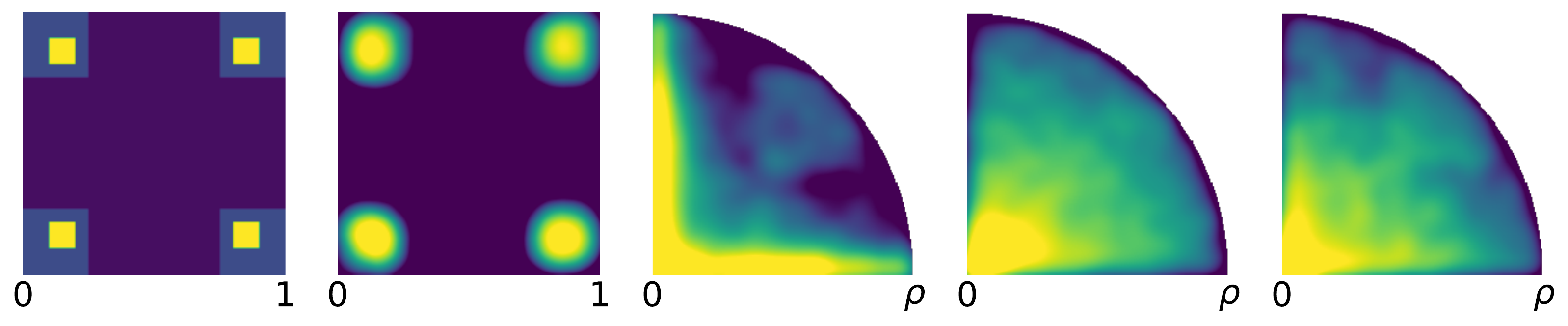}
    \vspace{-0.8cm}
    \caption{\textbf{(a)} Reward density in $[0, 1]^2$. \textbf{(b)} KDE fit on terminating states of the models trained with TB, $\rho=0.25$. \textbf{(c)} KDE fit on samples from the reward, brought back to $D_0$ using a \textbf{uniform} $P_B$, corresponding to what $P_F(s_0, -)$ needs to be in order to satisfy DB or TB. A richer search space for the densities $p_F(s, -)$ is required to fit this distribution. \textbf{(d)} $P_F(s_0, -)$ for a trained model with learnable $P_B$. \textbf{(e)} The measure induced by a trained $P_B$ on $D_0$, which matches the learned $P_F(s_0, -)$ in (d).}
    \label{fig:reward_kde_hypergrid}
\end{figure}

In \cref{fig:grid-structure-jsd}(b,c), we compare DB and TB on two versions of the environment ($\rho \in \{0.1, 0.25\}$), with both a uniform and a learned $P_B$, using the Jensen-Shannon divergence (JSD, \cref{sec:how-to-approximate-JSD}) between the learned terminating state distribution and the target distribution as an evaluation metric. The results confirm the findings of~\citet{malkin2022trajectory} on the discrete grid domain: the TB loss is more efficient in terms of credit assignment, as it learns to model the target distribution faster and more precisely than DB, and the environment with longer trajectory lengths is harder to model. Additionally, learning a backward policy significantly improves the learning curves of both methods. A justification of the importance of learning in a backward policy is provided in \cref{fig:reward_kde_hypergrid}(c,d,e). \Cref{fig:reward_kde_hypergrid}(b) shows a KDE plot fit on terminating states sampled from the model trained with TB on the $\rho=0.25$ domain. We provide more details in \cref{app:synthetic-continuous}.

\subsection{Low-dimensional stochastic control}
\label{sec:pis}

\begin{figure}[t]
\centering
\includegraphics[width=\linewidth]{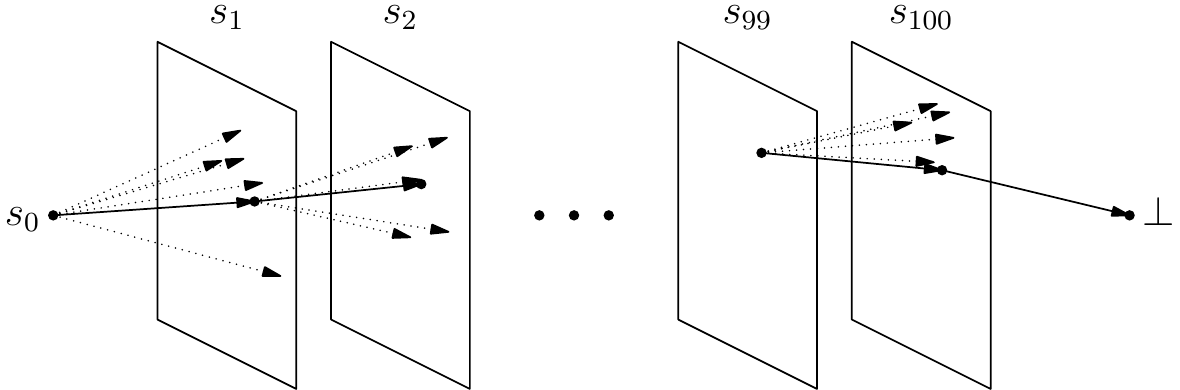}
\caption{The GFlowNet state space for stochastic control tasks. The solid arrows show a possible sampling trajectory and the dashed arrows show other possible actions, i.e., point to other states in the support of the reference kernel $\kappa$.}
\label{fig:sde_action_space}
\end{figure}

In this section, we show how generalized GFlowNets with a state space of a particular form can be used to learn (discretizations of) stochastic differential equations so as to sample from a black-box target density. We bridge two recent works: \citet{zhang2022path}, from which we borrow the datasets and many parts of the experimental setup, and \citet{malkin2022gfnhvi}, where various algorithms for training stochastic samplers in discrete spaces were considered and whose claims we validate in the continuous case.

We restate the problem considered by \citet{zhang2022path} in GFlowNet terms. A reward density is given on a Euclidean space $\R^n$ (e.g., the plane in \cref{fig:sde_action_space}). The state space is $\gS=\{s_0\}\cup(\R^n\times\{1,2,\dots,T\})$, where $T$ is the number of moves an agent will make before terminating (here, $T=100$). Thus the noninitial states are pairs $(\mathbf{x}_t,t)$ where $\mathbf{x}_t\in\R^n$ and $1\leq t\leq T$; we identify $s_0$ with $(\boldsymbol{0},0)$. Trajectories begin at $s_0$ and make successive steps through the copies of $\R^n$ until reaching the sink state.\footnote{To be precise, the reference measure $\nu$ is the Lebesgue measure on each copy of $\R^n$ and the counting measure on $\{s_0\}$. If $s_i$ is a state in the $i$-th copy of $\R^n$ (if $i>0$) or the initial state $s_0$ (if $i=0$), the reference kernel $\kappa(s_i,-)$ is Lebesgue on the $(i+1)$-st copy of $\R^n$ if $i<T$ and $\delta_\bot$ if $i=T$; $\kappa^b$ is defined similarly.} Learning a forward policy amounts to learning a conditional probability density $p(\mathbf{x}_{t+1}|\mathbf{x}_t,t)$ over $\R^n$. In particular, if this density is Gaussian, then the policy is the $T$-step Euler-Maruyama discretization of an It\^o stochastic differential equation (SDE).

\citet{zhang2022path} studied this problem in the case where the backward policy is fixed to be the discretization of a Brownian motion with fixed variance $\frac\sigma T$ pinned at $(\boldsymbol{0},0)$, and the forward policy is constrained to be Gaussian with the same variance $\frac\sigma T$ but with learned mean. (The theory of SDEs implies that in the $T\to\infty$ limit, the forward policy $P_F$ that minimizes the GFlowNet loss is indeed Gaussian with the same variance as the fixed $P_B$.) We thus aim to learn a function $\mu(\mathbf{x}_t,t)$, the mean of the forward policy, so as to make the policy sample from the target reward density.

The path integral sampler (PIS) training objective proposed by \citet{zhang2022path} minimizes the reverse KL divergence between two measures over trajectories: that defined by $P_F$ and that defined by $R$ and $P_B$. By Theorem 1 of \citet{malkin2022gfnhvi} (the proof of which trivially generalizes to the continuous case), the gradient of this objective with respect to the parameters of $P_F$ is proportional, in expectation, to that of the TB gradient when trained on-policy. A key difference between PIS and the on-policy TB objective is that the latter does not require access to the gradient of the reward distribution, but treats it as a black box.

\begin{table}[t]
\caption{Log-partition function estimation bias using importance-weighted bound $B_{\rm RW}$ (mean and standard deviation over 10 runs). The \textbf{bold} value in each column shows the best result and all those statistically equivalent to it ($p>0.1$ under a Welch's $t$-test). Algorithms assuming access to the gradient of the reward (non-black-box) are shown for comparison. Rows marked with $^*$ require importance weighting for gradient estimation. Cells with -- were unstable to optimize. Last three rows from \citet{zhang2022path}.}
\centering
\resizebox{1\linewidth}{!}{
\begin{tabular}{@{}c@{}llcc}
\toprule
Black box?&&&Gaussian mixture&Funnel\\\midrule
$\checkmark$ & Off-policy
& GFlowNet TB & ${\bf -0.003}\pm0.011$ & ${\bf -0.026}\pm0.020$ \\
$\checkmark$ & Off-policy & Reverse KL$^*$ & $-1.609\pm0.546$  & -- \\
$\checkmark$ & Off-policy & Forward KL$^*$  & ${\bf -0.001}\pm0.013$  & $-0.087\pm0.081$ \\\midrule
$\checkmark$ & On-policy
& GFlowNet TB & $-1.301\pm0.434$ & ${\bf -0.012}\pm0.108$ \\
$\checkmark$ & On-policy & Reverse KL & $-1.237\pm0.413$  & $-0.040\pm0.023$ \\
$\checkmark$ & On-policy & Forward KL$^*$  & $-0.007\pm0.023$  & $-0.034\pm0.143$ \\ \midrule
$\checkmark$ & Non-SDE & SMC
& $-0.362\pm0.293$ & $-0.216\pm0.157$
\\\midrule\midrule
 $\times$ & On-policy
 & PIS-NN & $-1.192\pm0.482$  & $-0.018\pm0.020$ \\\midrule
$\times$ & Non-SDE & HMC
 & $-1.876\pm0.527$ & $-0.835\pm0.257$ \\
\bottomrule
\end{tabular}
}
\label{tab:continuous_control}
\vspace*{-1em}
\end{table}  

\textbf{Datasets, algorithms, and baselines.} We evaluate GFlowNets and baselines on two synthetic densities: a 2-dimensional mixture of 9 Gaussians and the 10-dimensional funnel from MCMC literature \citep{nuts}. 

In addition to GFlowNet TB, we evaluate the two algorithms for minimizing divergences between trajectory measures studied by \citet{malkin2022gfnhvi}: the reverse KL optimized via policy gradient -- equivalent in expectation to TB -- and the forward KL, for which gradient estimation requires importance weighting. 
We also evaluate the algorithms in an off-policy setting, where the training trajectories are sampled with additional variance injected into the policy to encourage exploration (see \cref{app:experimental-details} for details). We include baselines from \citet{zhang2022path} as well.

All algorithms use the same model architecture as the PIS baseline for $\mu(\mathbf{x}_t,t)$ and are evaluated using two metrics as defined in \citet{zhang2022path}: the log-partition function estimation bias using simple and importance-weighted variational bounds, as defined in \cref{app:pis}.

\textbf{Results and discussion.} From the results in \cref{tab:continuous_control}, and the extended results in \cref{tab:continuous_control_app}, we conclude that the two main observations of \citet{malkin2022gfnhvi} continue to hold in this continuous setting. First, as expected, on-policy TB and reverse KL perform similarly when both can be stably optimized. Second, in settings where off-policy exploration is important, TB is more stable and achieves a better fit to the target than the other objectives, which require importance weighting for gradient estimation. \cref{fig:gaussian_modes} shows that exploration is necessary to discover modes. Finally, we note that TB is competitive with the PIS objective despite not having access to gradients of the reward density.

\iftoggle{arxiv}{\input{torusarxiv.tex}}{\textbf{Stochastic control on a torus.} In  \cref{sec:continuous-torus} we use a variant of the samplers considered in this section to model reward densities on the surface of a 2D torus, such as those over pairs of torsion angles in molecular conformations. The policies are parametrized as mixtures of von Mises distributions, illustrating that GFlowNet algorithms flexibly transfer to arbitrary spaces over which probability densities can be parametrically described. We show density plots in \cref{fig:torus_results}.}

\subsection{Posterior over continuous parameters in Bayesian structure learning}
\label{sec:daggfn}
To show the capacity of GFlowNets to model a distribution over a mixed space of discrete and continuous quantities, we study here the problem of learning the structure of a Bayesian network and its parameters, from a Bayesian perspective. Extending the work of \citet{deleu2022bayesian}, our goal here is to approximate the (joint) posterior distribution $P(G, \theta \mid \gD)$ over the directed acyclic graph (DAG) structure $G$ of the Bayesian Network (discrete component) and the parameters $\theta$ of its conditional probability distributions (continuous component), given a dataset of observations $\gD$.

We use a GFlowNet that is structured as follows: starting from the empty graph, the DAG $G$ is first generated one edge at a time, following the structure of DAG-GFlowNet \citep{deleu2022bayesian}. Once the graph $G$ has been completely generated, we then sample the parameters $\theta$ associated to it, in order to reach a valid terminating state $(G, \theta)$. Details about the state space, and the forward transition probability are given in \cref{app:bayesian-structure-learning}. 
We use the subtrajectory balance loss \citep{madan2022learning} to train the GFlowNet with $R(G, \theta) = P(\gD\mid  \theta, G)P(\theta, G)$ as a reward function.

\begin{table}[t]
    \centering
    \caption{Comparison between GFlowNet and other methods based on variational inference on the Bayesian structure learning task. (Graphs) RMSE between the estimated edge marginals and the exact edge marginals. (Params.) Average negative log-probability of the parameter samples under the exact posterior $P(\theta \mid G, \gD)$.}
    \label{tab:dag-gfn-comparison}
    \resizebox{1\linewidth}{!}{
    \begin{tabular}{llccc}
        \toprule
         & & \multicolumn{3}{c}{Number of variables ($d$)}\\\cmidrule(lr){3-5}
         & & 3 & 4 & 5\\
        \midrule
        \multirow{3}{*}{\rotatebox{90}{Graphs}} & BCD Nets & -- & $2.13 \times 10^{-1}$ & $2.61 \times 10^{-1}$ \\
        & DiBS & $3.28 \times 10^{-1}$ & $2.95 \times 10^{-1}$ & $3.15 \times 10^{-1}$ \\
        & GFlowNet & $\mathbf{1.50 \times 10^{-2}}$ & $\mathbf{1.61 \times 10^{-2}}$ & $\mathbf{1.80 \times 10^{-2}}$ \\
        \midrule
        \multirow{3}{*}{\rotatebox{90}{Params.}} & BCD Nets & -- & $\phantom{-}2.17 \times 10^{2}$ & $\phantom{-}2.63 \times 10^{2}$ \\
        & DiBS & $\phantom{-}5.87 \times 10^{2}$ & $\phantom{-}1.12 \times 10^{3}$ & $\phantom{-}2.12 \times 10^{3}$ \\
        & GFlowNet & $\mathbf{-1.75 \times 10^{0}}$ & $\mathbf{-3.06 \times 10^{0}}$ & $\mathbf{-5.17 \times 10^{0}}$ \\
        \bottomrule
    \end{tabular}
    }
\end{table}

In order to evaluate our approximation against the target distribution, we consider problems where the true posterior $P(G, \theta \mid \gD)$ may be computed in closed form. More precisely, we assume that the Bayesian network follows a linear-Gaussian model and that the number of random variables $d \leq 5$. Additional details about the experimental settings and metrics are available in \cref{app:bayesian-structure-learning}. In \cref{tab:dag-gfn-comparison}, we compare the performance of the GFlowNet with two baseline methods based on variational inference: DiBS \citep{lorch2021dibs} and BCD Nets \citep{cundy2021bcdnets}. In \cref{tab:dag-gfn-comparison} (top), we report the root mean-square error (RMSE) between the edge marginals computed with the approximation and the exact posterior $P(G\mid \gD)$; we observe that the model learned by the GFlowNet is significantly more accurate on the discrete component, supporting the observation made in \citet{malkin2022gfnhvi}. Moreover, in \cref{tab:dag-gfn-comparison} (bottom), we observe that the sampled $\theta$ from the GFlowNet are significantly more likely under the exact posterior $P(\theta \mid G, \gD)$, suggesting that the GFlowNet's approximation of the continuous component is also more accurate.

\subsection{Connections with diffusion models}
\label{sec:diffusion}

\begin{wraptable}{r}{0.25\textwidth} 
\centering
\vspace{-0.8cm}
\caption{ImageNet-$32$ results.}
\label{tab:ddpm_result}
\begin{small}
    \begin{tabular}{l|cc}
    \toprule
      Method  & FID$\downarrow$ & NLL$\downarrow$ \\
    \midrule
       Baseline & $17.65$ & $4.57$ \\
       MLE-GFN & $16.36$ & $4.47$ \\
    \bottomrule
    \end{tabular}
\end{small}
\end{wraptable}

We show how the generalized GFlowNet framework can be applied \textit{beyond} the setting of fitting a sampler to a target reward function. As shown in \citet{zhang2023unifying}, GFlowNets can also be trained to \emph{maximize likelihood} of a given set of terminating states with an algorithm called MLE-GFN. Here we apply MLE-GFN to generalized GFlowNets to improve denoising diffusion probabilistic models \citep[DDPMs;][]{ho2020ddpm}.

\paragraph{Sampling process.} The generative process in a DDPM can be seen as a special case of the sampling process in a generalized GFlowNet of the same form as in \cref{sec:pis} and \cref{fig:sde_action_space}. A fixed number of steps $T$ is made through a sequence of copies of a high-dimensional space $\R^n$ (with the $i$-th state in the trajectory representing, e.g., an image at noise level $T-i$). The policy at the first step, from $s_0$ to $(\mathbf{x}_1,1)$, is constrained to be unit Gaussian, while subsequent steps are conditional Gaussians with a known variance.

More specifically, recall that diffusion models begin with a sample $\mathbf{x}_1$ from a noise distribution and transform it through a sequence of conditional Gaussian steps $x_1\rightarrow x_2\rightarrow\dots\rightarrow x_T$ (note the unconventional reversed and one-based indexing). Viewing the intermediate samples $\mathbf{x}_t$ as states $(\mathbf{x}_t,t)$, we can cast sampling from the diffusion model as sampling from a GFlowNet, where the first action samples $\mathbf{x}_1$ by transitioning from the abstract initial state $s_0$ and subsequent actions follow a forward kernel $P_F(-\mid(\mathbf{x}_t,t))$ whose support is $\{(\mathbf{x},t+1):\mathbf{x}\in\mathbb{R}^n\}$ and whose density is a Gaussian conditioned on $\mathbf{x}_t$ and $t$.

\paragraph{Noising process.} While DDPMs typically fix the noising process -- corresponding to the backward policy $P_B$ in the GFlowNet -- and learn only the denoiser (forward process), MLE-GFN allows learning both $P_F$ and $P_B$ as Gaussian policies. The description and proof of soundness of MLE-GFN, as well as details of the parametrization of means and variances, can be found in \citet{zhang2023unifying}.

\paragraph{Experimental result.} We train a GFlowNet as described above on the ImageNet-$32$ dataset (treated as a set of terminating states) with $T=100$ steps. %
Table~\ref{tab:ddpm_result}
demonstrates the efficacy of our method compared to the DDPM baseline in terms of both the sample quality (FID) and density modeling (NLL).
We defer other details and example images to \cref{appendix:ddpm}.

\section{Conclusion}

We have developed a theory for generalized GFlowNets and illustrated it through experiments. Future work will exploit this theory and scale the experiments up to more complex and high-dimensional spaces where generation includes both discrete and continuous choices. Possible application areas include estimation of Bayesian neural network posteriors, molecular conformer generation (discussed in \cref{sec:continuous-torus}), and simulation-based inference for inverse problems in the natural sciences.

\section*{Acknowledgments}

The authors acknowledge funding from CIFAR, IVADO, Genentech, Samsung, and IBM.

N.M. thanks Alexander Tong and Yatin Dandi for useful discussions.

A.V. thanks Luca Thiede and Santiago Miret for their help with the experiments with the alanine dipeptide molecule.

T.D. thanks Mizu Nishikawa-Toomey and Jithendaraa Subramanian for their help and useful discussions about the Bayesian structure learning experiments.

\section*{Author contributions}
S.L., T.D., N.M. developed the theory and L.N.E. contributed some proofs. Experiments in \S\ref{sec:experiments} were done by S.L. (synthetic grid), N.M. (stochastic control), A.V. and A.H. (torus), T.D. (Bayesian structure learning), D.Z. (diffusion). Experiments on the importance of modeling assumptions were also done by P.L. Y.B. conceived and guided the project. All authors contributed to designing the experiments and writing the paper.

\bibliography{ref}
\bibliographystyle{style/icml2023}

\newpage 
\onecolumn
\appendix

\counterwithin{figure}{section}
\counterwithin{table}{section}

\input{appendix.tex}

\end{document}

%% file: style/math_commands.tex
\usepackage{amsmath,amsfonts,bm}

\def\gB{{\mathcal{B}}}

\def\gD{{\mathcal{D}}}

\def\gG{{\mathcal{G}}}

\def\gL{{\mathcal{L}}}

\def\gN{{\mathcal{N}}}

\def\gP{{\mathcal{P}}}

\def\gS{{\mathcal{S}}}
\def\gT{{\mathcal{T}}}
\def\gU{{\mathcal{U}}}

\def\gX{{\mathcal{X}}}

\def\sN{{\mathbb{N}}}

\def\sR{{\mathbb{R}}}

\def\indicator{{\mathbbm{1}}}

\newcommand{\E}{\mathbb{E}}

\newcommand{\R}{\mathbb{R}}

\newcommand{\ra}{{\rightarrow}}

\newcommand{\BlackBox}{\rule{1.5ex}{1.5ex}}  
\ifdefined\proof
    \renewenvironment{proof}{\par\noindent{\bf Proof\ }}{\hfill\BlackBox\\[2mm]}
\else
    \newenvironment{proof}{\par\noindent{\bf Proof\ }}{\hfill\BlackBox\\[2mm]}
\fi
\newtheorem{example}{Example} 
\newtheorem{theorem}{Theorem}
\newtheorem{lemma}{Lemma} 
\newtheorem{proposition}{Proposition}

\newtheorem{definition}{Definition}

%% file: torusarxiv.tex
\subsection{Stochastic control on a torus}
\begin{figure}[t]
    \centering
    \includegraphics[width=0.48\linewidth]{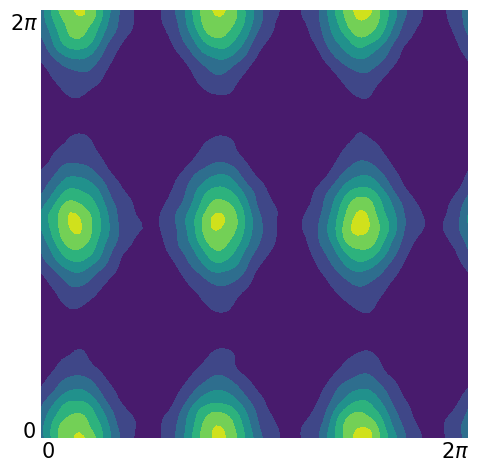}
    \includegraphics[width=0.48\linewidth]{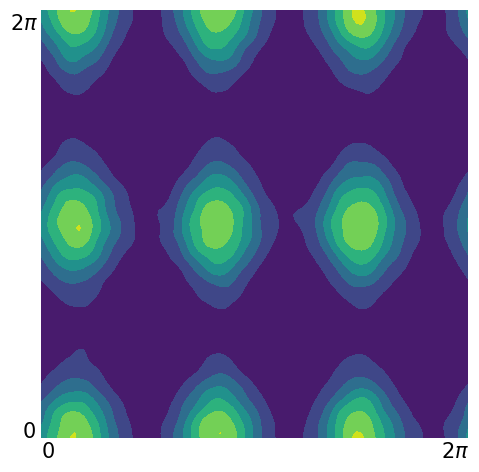}\\
    \includegraphics[width=0.48\linewidth]{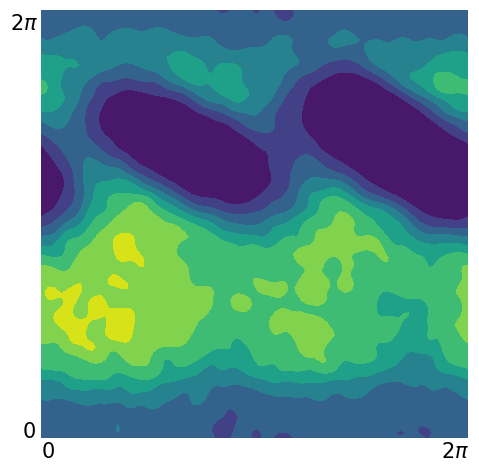}
    \includegraphics[width=0.48\linewidth]{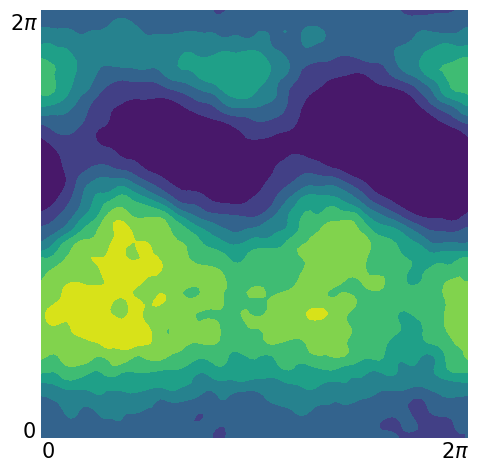}
    \caption{KDEs fit on samples from the reward functions (\textbf{a}: synthetic multimodal reward function, \textbf{c}: Boltzmann distribution of principal torsion angles of the alanine dipeptide molecule -- details in \cref{sec:continuous-torus}) and on samples from the corresponding trained GFlowNets (\textbf{b, d}). The topology of the torus imposes periodic boundary conditions on $[0, 2\pi)^2$.} 
    \vspace*{-4mm}
\label{fig:torus_results_arxiv}
\end{figure}

We consider a variant of the samplers discussed in \cref{sec:pis} to model reward densities on the surface of a 2D torus. Distributions over tori are useful to model torsion angles in molecular conformations, as we illustrate in \cref{sec:continuous-torus} and \cref{fig:molecule} with the alanine dipeptide molecule.

To model the surface of a torus, the measurable pointed graph is defined by $\gS = \{s_0\} \cup [0, 2\pi)^2 \times \{t \in \mathbbm{N}, 1 \leq t \leq T\}$, where $t$ denotes the step number and $T$ the trajectory length, and $s_0 = ((0, 0), 0)$. Note that here $[0,2\pi)$ has the topology of the circle, not that induced from the real line.

We consider two reward densities: a synthetic multimodal density, and a density based on the energy $\mathcal{E}$ of the alanine dipeptide molecule as a function of two of the angles defining the conformation of the molecule. More details are provided in \cref{sec:continuous-torus}.
 We provide a visual representation of learned and reward distributions in \cref{fig:torus_results_arxiv}.

%% file: appendix.tex
\section{Transition kernels: Additional notations and definitions}
\label{appendix:transition-kernels}
We first recall the definition of transition kernels and Markov kernels
\begin{definition}[Transition kernel]
    Let $(\bar{\gS}, \Sigma)$ be a measurable (state) space. A function $\kappa: \bar{\gS} \times \Sigma \rightarrow [0, +\infty)$ is called a positive $\sigma$-finite \emph{transition kernel} if
    \begin{enumerate}
        \item For any $B\in\Sigma$, the mapping $s \mapsto \kappa(s, B)$ is measurable, where the space $[0, +\infty)$ is associated with the Borel $\sigma$-algebra $\gB([0, +\infty))$;
        \item For any $s\in\gS$, the mapping $B\mapsto \kappa(s, B)$ is a positive $\sigma$-finite measure on $(\bar{\gS}, \Sigma)$.
    \end{enumerate}
    A transition kernel such that the mappings $B \mapsto \kappa(s, B)$ are probability measures is called a \emph{Markov kernel}.
    \label{def:transition-kernel}
\end{definition}

\paragraph{Notations.}
Given a measurable space $(\bar{\gS}, \Sigma)$, we denote by $\Sigma_{| \gU}$ the restriction of $\Sigma$ to any subset $\gU$ of $\bar{\gS}$.

\subsection{Products of kernels.} Given a measurable space $(\bar{\gS}, \Sigma)$, a positive measure $\nu$ on $(\bar{\gS},\Sigma)$, and a transition kernel $\kappa$ on $(\bar{\gS}, \Sigma)$, we denote by $\nu\kappa$ (resp.\ $\nu \otimes \kappa$) the measure on $(\bar{\gS}, \Sigma)$ (resp.\ $(\bar{\gS} \times \bar{\gS}, \Sigma \otimes \Sigma)$) defined for $B \in \Sigma$ (resp.\ $B \in \Sigma \otimes \Sigma$) as:
\begin{align}
    &\nu\kappa(B) = \int_{\bar{\gS}}\nu(ds)\kappa(s, B).
    \label{eq:product-measure-kernel} \\
    &\nu \otimes \kappa (B) = \iint_{\bar{\gS}^2} \indicator_B(s, s') \nu(ds) \kappa(s, ds') \label{eq:product-of-measures}
\end{align}
In particular, for any state $s \in \bar{\gS}$, the \emph{$n$-step measure} $\kappa^{n}(s, -)$ is recursively defined by $\kappa^0(s, -) = \delta_s$, the Dirac at $s$, and:
\begin{equation}
    \kappa^{n+1}(s, - ) = \kappa^n(s, -)\kappa.\label{eq:n-step-measure}
\end{equation}

The following lemma, proved in \cref{appendix:proofs-of-appendices} ensures that absolute continuity between transition kernels transfers to $n$-step measures
\begin{lemma}
Let $P_F$ be a transition kernel on $(\bar{\gS}, \Sigma)$ such that $P_F(s, -) \ll \kappa(s, -)$ for every $s \in \bar{\gS}$.
 Then for every $n \geq 1$, $P_F^n(s, -)$ and $s \in \gS$ is absolutely continuous wrt. $\kappa^n(s, -)$.
\label{lemma:absolute-continuity-powers}
\end{lemma}

\subsection{Equality between measures}
Given two measures $p$ and $q$ on $(\bar{\gS}, \Sigma)$, and a function $g: \bar{\gS} \ra \R$, we use the notation:
\begin{align*}
    p(ds) = g(s) q(ds)
\end{align*}
to say that for any measurable bounded function $f: \bar{\gS} \ra \R$:
\begin{align}
    \int_{\bar{\gS}} f(s) p(ds) = \int_{\bar{\gS}} f(s) g(s) q(ds).
\end{align}
Throughout the paper, we use the two notations interchangeable when the context allows it. Our proofs rely mostly on writing the equality with measurable bounded functions.

Equalities between product measures require a special care, especially when using both the kernel $\kappa$ and the backward reference kernel $\kappa^b$. For example \cref{eq:backward-kernel} means that for any measurable bounded function $f: \bar{\gS} \times \bar{\gS} \ra \R$ satisfying $f(s_0, s_0) = f(\bot, \bot) = 0$:
\begin{align*}
    \iint_{\bar{\gS} \times \bar{\gS}} f(s, s') \nu(ds) \kappa(s, ds') = \iint_{\bar{\gS} \times \bar{\gS}} f(s, s') \nu(ds') \kappa^b(s', ds).
\end{align*}
We choose not to write \cref{eq:backward-kernel} with the notation $\nu(ds)\kappa(s, ds') = \nu(ds') \kappa(s', ds)$ as it does not convey any information about when ``$ds$'' and ``$ds'$'' represent $\{s_0\}, \{s_0\}$ or $\{\bot\}, \{\bot\}$.

\subsection{Trajectory measures}
\begin{definition}
    Let $P_F$ be a transition kernel on $(\bar{\gS}, \Sigma)$. For any $n\geq 0$ and $s \in \bar{\gS}$, $P_F$ induces a measure $P_F^{\otimes{n}}(s, -)$ over the product space $(\bar{\gS}^{n+1}, \Sigma^{\otimes (n+1)})$. $P_F^{\otimes{n}}(s, -)$, called the \emph{$n$-step trajectory measure at $s$} recursively defined by
    \begin{align}
    P_F^{\otimes{0}}(s, -) &= \delta_s \label{eq:n-step-trajectory-measure-initial-condition},\\
    P_F^{\otimes{n+1}}(s, -) &= P_F^{\otimes{n}}(s, -) \otimes P_F.
        \label{eq:n-step-trajectory-measure-recursion}
    \end{align}
    \label{def:trajectory-measures}
\end{definition}

\paragraph{Notation.}
We use $\overrightarrow{s_{1:n}}$ to denote $(s_1, \dots, s_n)$ and $\overrightarrow{ds_{1:n}}$ to denote $ds_1\dots ds_n$.

We can write for example:
$P_F^{\otimes 1}(s, ds'\,ds_1) = \delta_{s}(ds') P_F(s', ds_1)$ and
$P_F^{\otimes 2}(s, ds'\,ds_1\,ds_2) = \delta_{s}(ds') P_F(s', ds_1) P_F(s_1, ds_2)$, and more generally:
\begin{equation*}
    P_F^{\otimes n}(s, ds'\,\overrightarrow{ds_{1:n}}) = P_F^{\otimes n-1}(s, ds'\,\overrightarrow{ds_{1:n-1}}) P_F(s_{n-1}, ds_n)
\end{equation*}

\subsection{Terminating state measure}
\label{sec:terminating-state-measure}
Given a measurable pointed DAG $G=(\bar{\gS}, \gT, \Sigma, s_0, \bot, \kappa, \kappa^b, \nu)$, any transition kernel $P_F$ on $(\bar{\gS}, \Sigma)$ induces a terminating state measure $P_{\top}$, which is the sum of the $n$-step terminating state measures defined as follows:
\begin{definition}
    Let $P_F$ be a transition kernel on $(\bar{\gS}, \Sigma)$. For any $n\geq 0$ we define the \emph{$n$-step terminating state measure} $P_{\top}^n$ over $(\gX, \Sigma_{| \gX})$, for any $B \in \Sigma_{| \gX}$ as:
    \begin{equation}
        \label{eq:n-step-terminating-state-measure}
        \hspace{-1em}P_{\top}^n(B) = \int_{\bar{\gS}^{n+1}} P_F^{\otimes{n}}(s_0, ds_1 \dots ds_{n+1}) \indicator_B(s_n) \indicator_{s_{n+1} = \bot}.
    \end{equation}
    The terminating state measure is defined as:
    \begin{equation}
        \label{eq:terminating-state-measure}
        P_\top: B \in \Sigma_{| \gX} \mapsto \sum_{n=1}^\infty P_\top^n(B)
    \end{equation}
    \label{def:terminating-states-measures}
    \vspace*{-1em}
\end{definition}
    
The following lemma, proved in \cref{appendix:proofs-of-appendices}, relates the $n$-step terminating measures to the $n$-step measures $P_F^n(s_0, -)$:
\begin{lemma}
    Let $P_F$ be a transition kernel on $(\bar{\gS}, \Sigma)$. For every $n \geq 1$, we have:
    \begin{equation}
        \label{eq:terminating-state-dist-PF}
        P_{\top}^n (dx) = P_F(x, \{\bot\}) P_F^{n-1}(s_0, dx)
    \end{equation}
    \label{lemma:terminating-state-dist-PF}
    \vspace*{-2em}
\end{lemma}

\subsection{Backward reference kernels}
\label{appendix:backward-kernels}
Given a reference kernel $\kappa$, designing a backward reference kernel $\kappa^b$ can be done using reverse kernels~\citep{cappe2009inferencehmm}, which we redefine below:
\begin{definition}[Reverse kernel]
    Let $(\bar{\gS}, \Sigma)$ be a measurable space, $\kappa$ be a transition kernel on $(\bar{\gS}, \Sigma)$, and $\nu$ be a positive measure on $(\bar{\gS}, \Sigma)$. A \emph{reverse kernel} $\kappa^{r}_{\nu}$ associated to $\nu$ and $\kappa$ is a transition kernel over $(\bar{\gS}, \Sigma)$ such that:
    \begin{align}
        \nu \otimes \kappa = (\nu \kappa) \otimes \kappa_\nu^r
        \label{eq:reverse-kernel}
    \end{align}
    \label{def:reverse-kernel}
\end{definition}

Note how \cref{eq:reverse-kernel} is different from the condition of the backward reference kernel \cref{eq:backward-kernel}. Conveniently, there is a reference measure $\nu$ for which the two conditions are equivalent, meaning that the backward reference kernel can be defined as the reverse kernel associated to $\nu$ and $\kappa$. While there is no guarantee that the reverse kernel exists or is unique in general, existence is guaranteed if $(\bar{\gS}, \gT)$ is a Polish space (e.g. a discrete space, the Euclidian space $\R^n$, hyperrectangles or balls in $\R^n$, or products or disjoint unions of countable families thereof) \citep{cappe2009inferencehmm}. The following proposition, proved in \cref{appendix:proofs-of-appendices} shows that.

\begin{proposition}
\label{prop:existence-of-backward-kernel}
Given a Polish space $(\bar{\gS}, \gT)$, with source and sink states $s_0, \bot \in \bar{\gS}$ such that $\{s_0\}$ and $\{\bot\}$ are both open and closed sets, and a transition kernel $\kappa$ on $(\bar{\gS}, \Sigma)$ satisfying \cref{eq:accessibility,eq:absorption,eq:finitely-absorbing}. Let $\nu$ be the measure defined by:
\begin{align}
    \nu = \sum_{n=0}^N \kappa^n(s_0, -),
    \label{eq:reference-measure-nu-prop}
\end{align}
and let $\kappa_\nu^r$ be any reverse kernel associated to $\kappa$ and $\nu$ that satisfies the following two conditions:
\begin{align}
&\kappa^r_\nu(s_0, - ) = 0 \quad \text{i.e. it's the trivial measure} \label{eq:reverse-absroption} \\
&\forall s \in \gS, \  \kappa^b(s, \{\bot\}) = 0 \label{eq:no-reverse-edge-to-bot}.
\end{align}
Let $\kappa^b$ be a transition kernel on $(\bar{\gS}, \Sigma)$ defined by:
\begin{align*}
    &\forall s \neq \bot, \ \kappa^b(s, -) = \kappa^r_\nu(s, -), \\
    &\forall B \in \Sigma_{|\gS}, \ \kappa^b(\bot, B) = \frac{1}{\nu(\{\bot\})} \nu \otimes \kappa(B \times \{\bot\}),\\
    & \kappa^b(\bot, \{\bot\}) = 0.
\end{align*}
$\nu$ is strictly positive and $\kappa^b$ satisfies \cref{eq:backward-kernel}. Note that the existence of a reverse kernel satisfying \cref{eq:reverse-absroption,eq:no-reverse-edge-to-bot} is guaranteed by \cref{lemma:existence-of-special-reverse-kernel} below, proved in \cref{appendix:proofs-of-appendices}.
\end{proposition}

\begin{lemma}
    \label{lemma:existence-of-special-reverse-kernel}
    Given a Polish space $(\bar{\gS}, \gT)$, with source and sink states $s_0, \bot \in \bar{\gS}$ such that $\{s_0\}$ and $\{\bot\}$ are both open and closed sets, and a transition kernel $\kappa$ on $(\bar{\gS}, \Sigma)$ satisfying \cref{eq:accessibility,eq:absorption,eq:finitely-absorbing}. Then the measure $\nu$ defined by:
    \begin{align}
        \nu = \sum_{n=0}^N \kappa^n(s_0, -),
        \label{eq:reference-measure-nu-lemma}
    \end{align}
    If $\kappa^b$ is a reverse kernel associated to $\nu$ and $\kappa$, then the kernel $\kappa'$ defined by:
    \begin{align}
        &\kappa'(s_0, -) = 0,\\
        &\kappa'(\bot, -) = \kappa^b(\bot, -), \\
        &\forall s' \in \gS\setminus \{s_0\}, \ \kappa'(s', \{\bot\}) = 0\\
        &\forall s' \in \gS\setminus \{s_0\}, \ \forall B \in \Sigma, \ \bot \notin B \Rightarrow \kappa'(s', B) = \kappa^b(s, B),
    \end{align}
    is also a reverse kernel associated to $\nu$ and $\kappa$.
\end{lemma}

\section{Pointed DAGs as measurable pointed graphs}
\label{appendix:pointed-graphs}
The following example shows that pointed directed acyclic graphs~\citep{bengio2021foundations} are a special case of finitely absorbing measurable pointed graphs.
\begin{example}
    Finite state spaces are special cases of measurable pointed graphs. Let $G = (V, E, s_{0}, \bot)$ be a pointed directed acyclic graph, where $V$ is the finite set of vertices, $E \subset V \times V$ is the set of directed edges, $s_{0}\in V$ is the initial state, and $\bot \in V$ is the sink state. 

    The set of vertices $V$ with the discrete topology corresponds to the state space. We can define a transition kernel $\kappa$ such that for any vertex $s\in V$, and any $B \in \gP(V)$, with $\gP(V)$ the power set of $V$, containing all subsets of $V$:
    \begin{equation*}
        \kappa(s, B) = \sum_{s'\in B}\indicator_{s{\rightarrow}s' \in E} + \indicator_{s = \bot, \bot \in B}
    \end{equation*}
    Using this transition kernel, the measure $B\mapsto \kappa^{n}(s, B)$ over $(V, \gP(V))$ counts the number of trajectories of length $n$ starting at $s$ that ends at a vertex in $B$ in the pointed graph $G$.
    
    The reverse kernel can be defined for any vertex $s' \in V$ and any $B \in \gP(V)$ as:
    \begin{equation*}
        \kappa^b(s', B) = \sum_{s \in B}\indicator_{s \ra s' \in E},
    \end{equation*}
    and the reference measure $\nu$ can be defined as the counting measure (that counts the number of elements in any $B \in \gP(V)$).
    
    Since $(V, \gP(V))$ is a discrete space, the condition of accessibility in \cref{eq:accessibility} can be verified for only singletons $B = \{s\}$. This condition then corresponds to having a positive number of trajectories of any length $n > 0$ starting at $s_{0}$ and ending in $s$, which is exactly the notion of accessibility in $G$. The continuity condition is trivially satisfied because the topology is discrete, and \cref{eq:backward-kernel} is trivially satisfied. Finally, \cref{eq:finitely-absorbing} is satisfied given the acyclicity of $G$.

    \label{ex:pointed-graph}
\end{example}

\section{Experimental details}
\label{app:experimental-details}

\subsection{Approximating the Jensen-Shannon Divergence}
\label{sec:how-to-approximate-JSD}
Given an unnormalized target reward measure $r$ with respect to the Lebesgue measure on a bounded space $\gX$, and a GFlowNet sampler $P_\top$ of terminating states, we approximate the JSD between the learned sampler and the normalized distribution $R(dx) = \frac{1}{\int_{\gX} r(x') dx'} r(x) dx$ as follows:
\begin{enumerate}[label=(\arabic*),left=0pt,nosep,]
    \item We sample $N$ points from the the target distribution using rejection sampling, with a uniform distribution as a proposal,
    \item We fit a kernel density estimator (KDE) on the above samples,
    \item We fit a second KDE on $N$ samples from $P_\top$
    \item We use both KDEs to score a fixed set of points defining a discretization of the sample space $\gX$,
    \item We normalize both sets of scores in order to obtain valid probability mass functions on the grid,
    \item We evaluate the JSD between the two probability mass functions.
\end{enumerate}

\subsection{A synthetic continuous environment}
\label{app:synthetic-continuous}
The forward and backward kernels $P_F, P_B$ are defined by their densities $p_F$ and $p_B$ wrt. the reference kernels $\kappa, \kappa^b$.

$\kappa, \nu, \kappa^b$ satisfy the requirements of a finitely absorbing measurable pointed graph. More notably, all states can be reached from $s_0$ within $1 + \left\lceil{\frac{\sqrt{2}}{\rho}}\right\rceil$ steps.

The topology $\gT$ on $\bar{\gS} = \gS \cup \{\bot\}$ is the disjoint union topology on $\{s_0, \bot\}$ and $\gS$. 

We parametrized $p_F(s_0, -)$ using a mixture of four Beta distributions for both the radius $r \in (0, \rho)$ and the angle $\theta \in (0, \frac{\pi}{2})$. We used a mixture of two Beta distributions for the angle $\theta \in (\theta_{min}(s), \theta_{max}(s))$ when modeling $p_F(s, -)$ and $p_B(s, -)$.  The forward policy neural network has an extra output head corresponding to the probability of terminating the trajectory, i.e. $p_F(s_0, \bot)$. The learned probabilities were effectively multiplied by the right Jacobians to account for the support of the Beta distributions ($[0, 1]$) being different from that of $\theta$ or $r$.

\paragraph{Reward density.}

The reward measure $R$ was specified using a density $r$ wrt. the Lebesgue measure $\lambda$ on $\gX = (0, 1)^2$. Following \citet{bengio2021flow} and \citet{malkin2022trajectory}, the (unnormalized) density is defined for every $x = (x_1, x_2) \in \gX$ as:
\begin{align*}
    r(x) = 0.1 + 0.5 \indicator_{|x_1 - 0.5| \in (0.25, 5]}\indicator_{|x_2 - 0.5| \in (0.25, 5]} + 2 \indicator_{|x_1 - 0.5| \in (0.25, 5]}\indicator_{|x_2 - 0.5| \in (0.3, 4)}
\end{align*}

\paragraph{Hyperparameters.}

We learned the concentration parameters of the Beta distributions, which were restrained to the interval $[0.1, 5.1]$, using a three-layered neural network with 128 units per layer, and leaky ReLU activation for $s \neq s_0$. The parameters corresponding to $p_F(s_0, -)$ were learned separately.

Each iteration consisted of sampling 128 trajectories from the forward policy, and evaluating the TB or the DB loss (with $\alpha=1$), before taking a gradient step on the learned parameters. We trained the models for 20,000 iterations.

For both the DB and TB losses, we used a learning rate of $10^{-3}$ for the parameters of $p_F, p_B, p_F(s_0, -)$ (and $\log Z$ for TB, $u$ for DB). The learning rate was annealed using a discount factor of $0.5$ every $2500$ iterations.

In experiments with learned $p_B$, both $p_F$ and $p_B$ shared parameters except in the output layer. In DB, $p_F$ and $u$ shared parameters except in the output layer.

\paragraph{Evaluation metric.}
We approximated the JSD between the learned the terminating state distribution and the target distribution following the scheme described in \cref{sec:how-to-approximate-JSD}.

\subsection{Low-dimensional stochastic control}
\label{app:pis}

The neural network computing $\mu(\mathbf{x}_t,t)$ had the same architecture as in \citet{zhang2022path}: a pair of 2-layer MLP processing $\mathbf{x}_t$ and a 128-dimensional Fourier feature representation of $t$, followed by a 3-layer MLP on the concatenation of the features derived from $\mathbf{x}_t$ and from $t$. We set $\sigma=5$ for the Gaussians density and $\sigma=1$ for funnel density. Exploration algorithms added a constant $\frac{\epsilon^2}{T}$ to the sampling policy variance at each step; we used a value of $\epsilon=0.1$ linearly annealed to 0 over the course of training. All models are trained for 1500 batches of 300 samples with a learning rate of $10^{-2}$ for the policy and $10^{-1}$ for $\log Z$ (in the case of GFlowNet algorithms); we found that higher learning rates made optimization unstable. We also observed that the off-policy forward KL and TB algorithms continue to improve with longer training, unlike the on-policy algorithms, which experience mode collapse and cease to discover new areas of the density landscape. \Cref{fig:gaussian_modes} shows samples from models trained with various algorithms and highlights the importance of exploration.

\begin{figure}[t]
\centering
\begin{tabular}{@{}c@{}c@{}c@{}c@{}}
& GFlowNet TB & Reverse KL & Forward KL 
\\
\raisebox{0.5in}{No exploration}
& \includegraphics[width=0.2\textwidth]{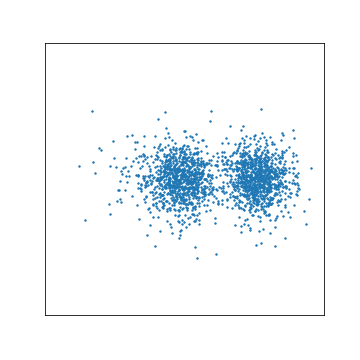}
& \includegraphics[width=0.2\textwidth]{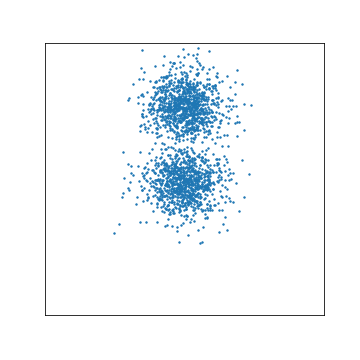}
& \includegraphics[width=0.2\textwidth]{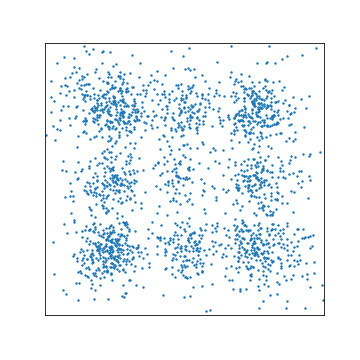}
\\
\raisebox{0.5in}{Fixed exploration $\epsilon=0.1$}
& \includegraphics[width=0.2\textwidth]{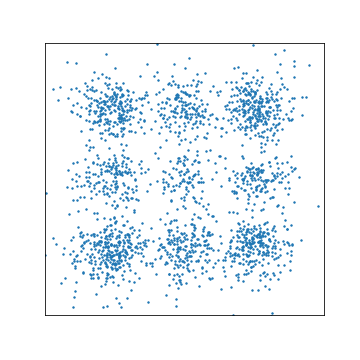}
& \includegraphics[width=0.2\textwidth]{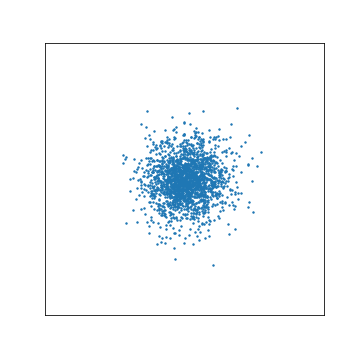}
& \includegraphics[width=0.2\textwidth]{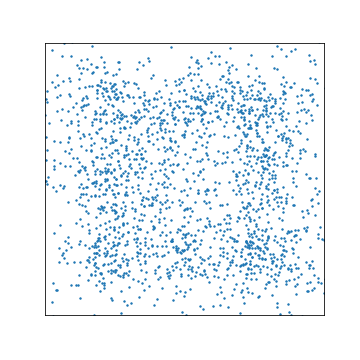}
\\
\raisebox{0.5in}{Annealed exploration $\epsilon=0.1\searrow0$}
& \includegraphics[width=0.2\textwidth]{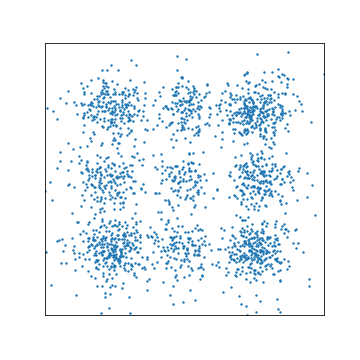}
& \includegraphics[width=0.2\textwidth]{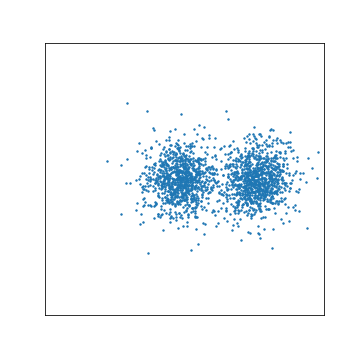}
& \includegraphics[width=0.2\textwidth]{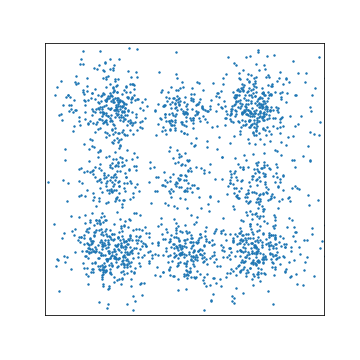}
\\
\raisebox{0.5in}{Target}
& \includegraphics[width=0.2\textwidth]{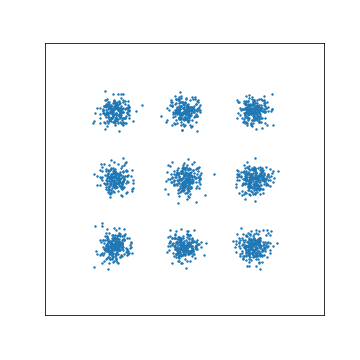}
\end{tabular}
\caption{The target density for 9 Gaussians and samples from models trained with various algorithms trained for 1500 batches. (When trained longer, GFlowNet TB policies with exploration learn to model the modes with higher precision.)}
\label{fig:gaussian_modes}
\end{figure}

The simple and importance-weighted estimates of the log-partition function from \citet{zhang2022path} are defined, in GFlowNet terms, as
\begin{align*}
B&=\frac1K\sum_{i=1}^K\log\frac{R(x_T^{(i)})p_B^{\otimes T}(\tau^{(i)}|x_T^{(i)})}{p_F^{\otimes T}(\tau^{(i)})},\\
B_{\rm RW}&=\log \frac1K\sum_{i=1}^K\frac{R(x_T^{(i)})p_B^{\otimes T}(\tau^{(i)}|x_T^{(i)})}{p_F^{\otimes T}(\tau^{(i)})},
\end{align*}
where the $\tau^{(i)}$ are $K$ trajectories sampled from $P_F$, the $x^{(i)}$ are their terminating states, and $p_F^{\otimes T}(\tau^{(i)}),p_B^{\otimes T}(\tau^{(i)}|x_T^{(i)})$ are the products of forward (resp.\ backward) Gaussian policy densities along the trajectories. 
Note that both estimates would equal the true integral of the reward density for a perfect sampler. Identically to \citet{zhang2022path}, we use $K=2000$ for the 2-dimensional Gaussian mixture dataset and $K=6000$ for the 10-dimensional funnel dataset.

We show extended results, including both simple and importance-weighted variational bounds, in \cref{tab:continuous_control_app}.

\begin{table*}[t]
\caption{Estimation bias of the log-partition function using simple ($B$) and importance-weighted ($B_{\rm RW}$) bounds (mean and standard deviation over 10 runs). The \textbf{bold} value in each column shows the best result and all those statistically equivalent to it ($p>0.1$ under a Welch's $t$-test). Algorithms assuming access to the gradient of the reward (non-black-box) are shown for comparison. Rows marked with $^*$ require importance weighting for gradient estimation. Cells with -- were unstable to optimize. Last three rows from \citet{zhang2022path}.}
\centering
\resizebox{1\linewidth}{!}{
\begin{tabular}{@{}c@{}llcccc}
\toprule
&&&\multicolumn{2}{c}{Gaussian mixture $(d=2)$}&\multicolumn{2}{c}{Funnel $(d=10)$}
\\\cmidrule(lr){4-5}\cmidrule(lr){6-7}
Black box?&&&$B$&$B_{\rm RW}$&$B$&$B_{\rm RW}$\\\midrule
$\checkmark$ & Off-policy
& GFlowNet TB & ${\bf -0.150}\pm0.019$ & ${\bf -0.003}\pm0.011$ & ${\bf -0.219}\pm0.020$ & ${\bf -0.026}\pm0.020$ \\
$\checkmark$ & Off-policy & Reverse KL$^*$ & $-1.706\pm0.537$ & $-1.609\pm0.546$ & -- & -- \\
$\checkmark$ & Off-policy & Forward KL$^*$ & $-0.306\pm0.036$ & ${\bf -0.001}\pm0.013$ & $-2.822\pm0.576$ & $-0.087\pm0.081$ \\\midrule
$\checkmark$ & On-policy
& GFlowNet TB & $-1.409\pm0.427$ & $-1.301\pm0.434$ & $-0.265\pm0.026$ & ${\bf -0.012}\pm0.108$ \\
$\checkmark$ & On-policy & Reverse KL & $-1.348\pm0.397$ & $-1.237\pm0.413$ & $-0.259\pm0.018$ & $-0.040\pm0.023$ \\
$\checkmark$ & On-policy & Forward KL$^*$ & $-0.254\pm0.032$ & $-0.007\pm0.023$ & $-1.384\pm0.284$ & $-0.034\pm0.143$ \\ \midrule
$\checkmark$ & Non-SDE & SMC
&\multicolumn{2}{c}{$-0.362\pm0.293$} & \multicolumn{2}{c}{$-0.216\pm0.157$}
\\\midrule\midrule
 $\times$ & On-policy
 & PIS-NN & $-1.691\pm0.370$ & $-1.192\pm0.482$ & $-0.098\pm0.005$ & $-0.018\pm0.020$ \\\midrule
$\times$ & Non-SDE & HMC
 & \multicolumn{2}{c}{$-1.876\pm0.527$} & \multicolumn{2}{c}{$-0.835\pm0.257$}
\\\bottomrule
\end{tabular}
}
\label{tab:continuous_control_app}
\end{table*}  

\subsection{Stochastic control on a torus environment}
\label{sec:continuous-torus}

\iftoggle{arxiv}{}{To model the surface of a torus, the measurable pointed graph is defined by $\gS = \{s_0\} \cup [0, 2\pi)^2 \times \{t \in \mathbbm{N}, 1 \leq t \leq T\}$, where $t$ denotes the step number and $T$ the trajectory length, and $s_0 = ((0, 0), 0)$. Note that here $[0,2\pi)$ has the topology of the circle, not that induced from the real line.}
The transition kernel $\kappa(s, -)$ for any $s = (\tilde{s}, t)$ when $t < T$ is the product of the Lebesgue measure on $[0, 2\pi)^2$ with the Dirac measure at $t + 1$, and $\kappa((\tilde{s}, T), -) = \delta_{\bot}$ for every $\tilde{s} \in [0, 2\pi)^2$. Similar to \cref{sec:pis}, the reference measure is $\nu = \delta_{s_0} + \sum_{t=1}^T \lambda \otimes \delta_t$, where $\lambda$ is the Lebesgue measure on each copy of the torus $[0, 2\pi)^2$.

We parameterized the densities $p_F$ and $p_B$ with mixtures of independent von Mises distributions defined by a measure of location $\mu$ and a measure of concentration $\kappa$. We considered two tasks in the torus environment defined by different reward functions.

\paragraph{Synthetic multimodal task.}  For this task, we designed a reward density with six modes on the torus surface:
\[
R_6(\psi, \varphi) = (\sin(3\psi) + \cos(2\varphi) + 2)^3.
\]

\paragraph{Molecule conformation task.}  
In this task, we define the reward function using the energy $\mathcal{E}$ of an alanine dipeptide molecule, which depends on the conformation of the molecule $\mathcal{C}$ (spatial arrangement of its atoms). This conformation can be efficiently parametrizeded using internal coordinates: bond lengths, bond angles, and torsion angles \citep{torsionaldiffusion2022, conformersearch2022luca}. For alanine dipeptide, there are four torsion angles largely influencing the energy (see \cref{fig:molecule}). In our experiments, the GFlowNet generates values for the angles $\psi$ and $\varphi$ while keeping all other coordinates fixed. In this way, the support of the reward function remains a torus, and its values are proportional to the Boltzmann distribution with energy $\mathcal{E}$:
\[
R_{AD}(\psi, \varphi) = \exp(-\mathcal{E}(\mathcal{C(\psi, \varphi)})),
\]

\iftoggle{arxiv}{The plots in \cref{fig:torus_results_arxiv} show the results of training a GFlowNet on a toroidal space with a continuous synthetic multimodal reward function $R_6$ (see text) and a reward function defined by the Boltzmann distribution of the alanine dipeptide molecule $R_{AD}$. 
The images represent the density over a discretization of the space $[0, 2\pi)^2$, obtained after fitting KDE with 100,000 samples. The samples to fit the reward densities (\cref{fig:torus_results_arxiv}(\textbf{a-c})) were obtained via rejection sampling, and the GFlowNet densities (\cref{fig:torus_results_arxiv}\textbf{b-d}) use GFlowNet samples from the learned distribution over terminating states $P_\top$. }{
\begin{figure*}[htb]
    \centering
     \begin{subfigure}[b]{0.23\textwidth}
        \centering
         \includegraphics[width=\textwidth]{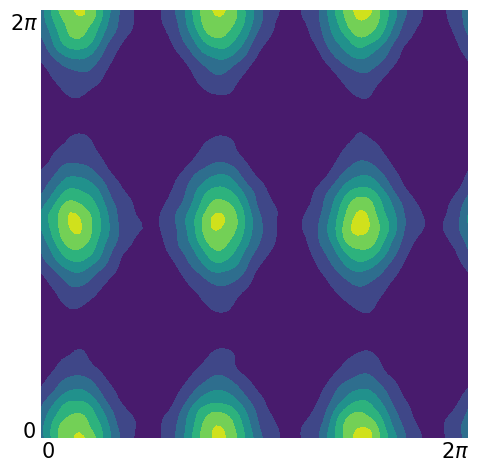}
         \caption{Synthetic reward}
         \label{fig:kde_reward_synth}
     \end{subfigure}
     \hfill
     \begin{subfigure}[b]{0.23\textwidth}
         \centering
         \includegraphics[width=\textwidth]{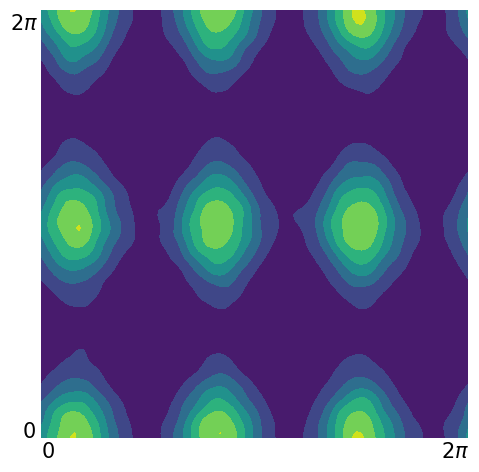}
         \caption{GFlowNet learned $P_\top$}
         \label{fig:kde_gfn_synth}
     \end{subfigure}
    \hfill
     \begin{subfigure}[b]{0.23\textwidth}
        \centering
         \includegraphics[width=\textwidth]{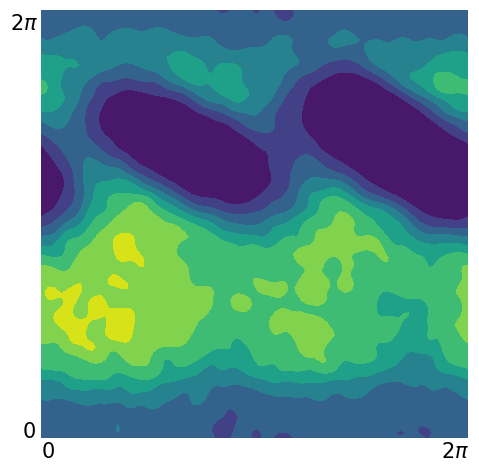}
         \caption{Molecular energy reward}
         \label{fig:kde_reward_molecule}
     \end{subfigure}
     \hfill
     \begin{subfigure}[b]{0.23\textwidth}
         \centering
         \includegraphics[width=\textwidth]{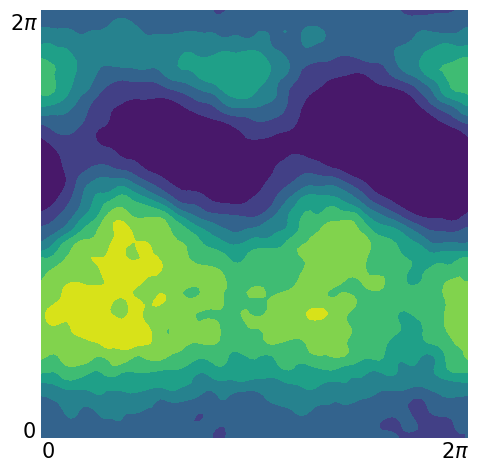}
         \caption{GFlowNet learned $P_\top$}
         \label{fig:kde_gfn_molecule}
     \end{subfigure}
\caption{Results of training a GFlowNet on a toroidal space with a continuous synthetic multimodal reward function $R_6$ (see text) and a reward function defined by the Boltzmann distribution of the alanine dipeptide molecule $R_{AD}$. 
The images represent the density over a discretization of the space $[0, 2\pi)^2$, obtained after fitting KDE with 100,000 samples. The samples to fit the reward densities (\cref{fig:kde_reward_synth} and \cref{fig:kde_reward_molecule}) were obtained via rejection sampling, and the GFlowNet densities (\cref{fig:kde_gfn_synth} and \cref{fig:kde_gfn_molecule}) use GFlowNet samples from the learned distribution over terminating states $P_\top$. Note that the topology of the torus imposes periodic boundary conditions on $[0, 2\pi)^2$.} 
\label{fig:torus_results}
\end{figure*}
}

\begin{figure*}[t]
\begin{minipage}[t]{0.48\textwidth}
    \centering
    \vspace{0pt}
    \includegraphics[width=\textwidth]{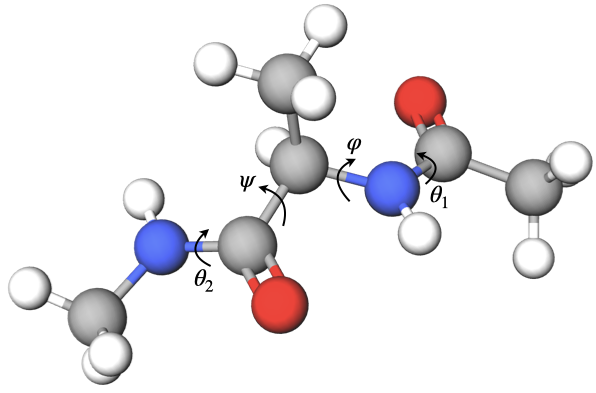}
\end{minipage}\hfill
\begin{minipage}[t]{0.48\textwidth}
    \centering
    \caption{Alanine dipeptide 3D structure. Torsion angles $\psi$, $\varphi$, $\theta_1$, $\theta_2$ have the biggest impact on the energy of the molecule. A pair of torsion angles $\varphi$ and $\psi$ can take any values $\in[0, 2\pi]$, while $\theta_1$ and $\theta_2$ can be either close to $0$ or $\pi$ due to energy barriers \citep{alaninedipeptidestudy}. The image is rendered using MolView \cite{molview}}
    \label{fig:molecule}
\end{minipage}
\end{figure*}

\paragraph{Results.} To evaluate the performance of the GFlowNet trained with the TB loss, we calculated the Jensen-Shannon divergence (see \cref{sec:how-to-approximate-JSD} for details about its estimation) between the learned terminating state distribution $P_\top$ and the normalized reward distribution. We provide a visual representation of learned and reward distributions in \iftoggle{arxiv}{\cref{fig:torus_results_arxiv}}{\cref{fig:torus_results}}. Quantitatively, the GFlowNet achieved a JSD of $0.063$ for the synthetic multimodal task and $0.009$ for the molecule conformation task. These results show that the generalized GFlowNet can model probability densities over non-Euclidean spaces.

\paragraph{Hyperparameters.} We modeled both $p_F$ and $p_B$ with 5-layer perceptrons with 512 hidden units per layer, training the full set of parameters of each model separately. These models output, for each angle $\psi$ and $\varphi$, the location $\mu_i$ and concentration $\kappa_i$ of 5 independent von Mises distributions, mixed with learned weights $w_i$. To take into account the topology of the torus, we encoded input angles with trigonometric transformations ($\sin(k\psi)$, $\cos(k\psi)$, $k=1, \dots, 5$, for both angles $\psi$ and $\varphi$). We used a learning rate of $10^{-5}$ for the model parameters and $10^{-2}$ for $\log Z$, updating the parameters with batches of 100 trajectories of length $T=10$. With the synthetic reward, the model converged in about 5,000 iterations; in the molecular conformation task, we trained for 40k iterations.

\subsection{Posterior over continuous parameters in Bayesian structure learning}
\label{app:bayesian-structure-learning}

\paragraph{Bayesian Networks.} Recall that a Bayesian Network is a probabilistic model, where the joint distribution over $d$ random variables $\{X_{1}, \ldots, X_{d}\}$ factorizes according to a directed acyclic graph (DAG) $G$ as:
\begin{equation*}
    P(X_{1}, \ldots, X_{d}; \theta, G) = \prod_{i=1}^{d} P(X_{i}\mid \mathrm{Pa}_{G}(X_{i}); \theta_{i}),
\end{equation*}
where $\mathrm{Pa}_{G}(X_{i})$ is the set of parents of $X_{i}$ in $G$, and $\theta_{i}$ is the set of parameters for the conditional probability distribution of $X_{i}$. We denote by $\theta = \{\theta_{1}, \ldots, \theta_{d}\}$ the set of all the parameters of this model.

We assume that $\theta \in \Theta_{G}$, where $\Theta_{G}$ is the space of all parameters for the Bayesian Network. Note that this space of parameters depends on the structure $G$ of the Bayesian Network. For example, the Bayesian Network where all the random variables are mutually independent (corresponding to $G$ being empty) has fewer parameters than another Bayesian Network that encodes dependencies between those random variables. We will also denote by $\gG$ the space of DAG over $d$ nodes; the number of elements in this space grows super-exponentially with $d$.

\paragraph{Linear Gaussian model.} In order to compute the exact posterior distribution $P(G, \theta\mid \gD)$ in closed-form, we consider here a linear-Gaussian model for the parametrization of the conditional probability distributions appearing in the Bayesian Network. More precisely, the conditional probability distribution is given by
\begin{align*}
P(X_{i}\mid \mathrm{Pa}_{G}(X_{i}); \theta_{i}) &= \mathcal{N}(X_{i}\mid \mu_{i}, \sigma^{2}) & & \mathrm{where} & \mu_{i} &= \sum_{X_{j}\in\mathrm{Pa}_{G}(X_{i})}\theta_{ij}X_{j}.
\end{align*}
In other words, $X_{i}$ follows a Normal distribution, whose mean $\mu_{i}$ is given by a linear combination of its parents, and with fixed variance $\sigma^{2}$. For this class of models,
\begin{equation*}
    \Theta_{G} \simeq \varprod_{i=1}^{d} \sR^{|\mathrm{Pa}_{G}(X_{i})|}
\end{equation*}

\paragraph{GFlowNet over a mixed state space.} We are using the GFlowNet in order to approximate the joint posterior distribution $P(G, \theta \mid \gD)$, and therefore its terminating states have the form $(G, \theta)$, where $G\in \gG$ is a DAG, and $\theta \in \Theta_{G}$ are the associated parameters. Unlike \citet{nishikawa2022bayesian}, which uses Variational Bayes to update the distribution over parameters $\theta$, we model the distribution over both the graphs and parameters using a single GFlowNet.

The generation of a terminating state follows 2 phases: during the first phase, the DAG $G$ is constructed by adding one edge at a time, starting from the empty graph, following the structure of DAG-GFlowNet \citep{deleu2022bayesian}. To reach a graph $G$ with $k$ edges, we therefore are taking $k$ steps in the GFlowNet. The states traversed during this first phase have no parameters associated to them; we denote by $(G, \sharp) \in \gS$ such an (intermediate) state, where $\sharp \notin \Theta_{G'}$ for any $G'\in \gG$ is a symbol indicating that $G$ has no corresponding parameters.

Then once we have finished adding edges (in practice, this decision is made by selecting a special ``stop'' action), we sample the parameters $\theta \in \Theta_{G}$ associated to $G$ by taking a final step in the GFlowNet to reach the terminating state $(G, \theta) \in \gX$. Since the space of parameters depends on the graph $G$, we define the state space of the GFlowNet as
\begin{align*}
    \gS &= \bigcup_{G\in \gG} \{G\} \times \bar{\Theta}_{G} && \textrm{and} & \gX &= \bigcup_{G\in \gG} \{G\} \times \Theta_{G},
\end{align*}
where $\bar{\Theta}_{G} = \Theta_{G} \cup \{\sharp\}$ indicates the space of parameters, augmented with the special symbol $\sharp$. All the states in this state space are guaranteed to be accessible from the initial state $(G_{0}, \sharp)$, where $G_{0}$ is the empty graph.

\paragraph{Reference kernel.} Given a DAG $G$, the measure $\kappa((G, \sharp), -)$ is the sum of a discrete measure (to transition to another intermediate state $(G', \sharp)$) and a continuous measure (to transition to a terminating state $(G, \theta)$). We can write this measure as
\begin{equation*}
    \kappa((G, \sharp), -) = \sum_{G' \in \mathrm{Ch}(G)}\delta_{(G', \sharp)} + (\delta_{G} \otimes \lambda_{\Theta_{G}}),
\end{equation*}
where $\mathrm{Ch}(G)$ represents the children of $G$ in DAG-GFlowNet \citep{deleu2022bayesian}, i.e. the graphs $G'$ obtained by adding an edge to $G$, and $\lambda_{\Theta_{G}}$ is the Lebesgue measure over $\Theta_{G}$. Moreover, we also have $\kappa((G, \theta), -) = \delta_{\bot}$ for all terminating state $(G, \theta) \in \gX$; in other words, there is no transition from a terminating state other than to the sink state.

The backward reference kernel $\kappa^{b}$ on the other hand is simpler: it is always a discrete transition kernel, regardless of the state. We have
\begin{align*}
    \kappa^{b}((G, \theta), -) &= \delta_{(G, \sharp)} && \textrm{and} & \kappa^{b}((G', \sharp), -) &= \sum_{G \in \mathrm{Pa}(G')}\delta_{(G, \sharp)},
\end{align*}
where $\mathrm{Pa}(G')$ are the parents of $G'$ in DAG-GFlowNet, i.e. they are the graphs obtained by removing a single edge from~$G'$.

\paragraph{Forward transition probability.} In order to define the $P_{F}$, we consider 2 cases: either we have a distribution of the form $P_{F}(G'\mid G)$, where $G'$ is the result of adding an edge to $G$, or a distribution of the form $P_{F}(\theta \mid G)$. Note that here we are using a slight abuse of notation, where $P_{F}(G'\mid G)$ (resp. $P_{F}(\theta\mid G)$) represents $P_{F}((G', \sharp) \mid (G, \sharp))$ (resp. $P_{F}((G, \theta)\mid (G, \sharp))$). The distribution $P_{F}(\theta \mid G)$ is parametrized by a Normal distribution, whose mean and (diagonal) covariance are returned by a neural network. Similar to \citep{deleu2022bayesian}, $P_{B}$ is fixed to the uniform distribution.

\paragraph{Data generation.} We sampled a dataset $\gD$ as follows: (1) we first generated a DAG $G^{*}$ from an Erd\"os-Renyi model, then (2) we sampled the parameters $\theta^{*}$ of the conditional probability distributions from a Normal distribution, each edge having a weight $\theta_{ij}^{*} \sim \gN(0, 1)$, and finally (3) we sampled $N=100$ datapoints from the Bayesian Network described above with $(G^{*}, \theta^{*})$, using ancestral sampling. Note that the ground-truths $G^{*}$ and $\theta^{*}$ are unknown to the GFlowNet, and it only uses the observations from $\gD$.

\paragraph{Evaluations.} To evaluate the quality of the approximation learned by the GFlowNet, and to compare it against the baseline methods based on variational inference \citep{lorch2021dibs,cundy2021bcdnets}, we study the distribution over graphs (discrete component) and the distribution over graphs (continuous component) separately. Recall that since we assume that our model is linear-Gaussian over small graphs ($d \leq 5$), we can compute the exact posterior distribution $P(G, \theta\mid \gD)$ in closed form. 

In \cref{tab:dag-gfn-comparison}, we compared the edge marginals estimated using the posterior approximations to the exact edge marginals. For any pair of random variables $(X_{i}, X_{j})$, this means evaluating the following marginals
\begin{align*}
    P(X_{i}\rightarrow X_{j} \mid \gD) &= \sum_{G\mid X_{i} \in \mathrm{Pa}_{G}(X_{j})}P(G\mid \gD).
\end{align*}
To estimate this marginal using samples $\{(G_{k}, \theta_{k})\}_{k=1}^{K}$ from the GFlowNet (or from the variational inference methods), we can simply use the sample graphs $G_{k}$ in order to get an empirical approximation of the maginal posterior $P(G\mid \gD)$. In other words,
\begin{equation*}
    \hat{P}(X_{i}\rightarrow X_{j}) = \frac{1}{K} \sum_{k=1}^{K} \mathbbm{1}(X_{i} \rightarrow X_{j} \in G_{k})
\end{equation*}
We then report the root mean-square error (RMSE) between the edge marginals estimated using the posterior approximations, and those computed using the exact posterior:
\begin{equation*}
    \mathrm{RMSE}(\hat{P}, P) = \left(\frac{1}{d(d-1)} \sum_{i\neq j} \big(\hat{P}(X_{i} \rightarrow X_{j}) - P(X_{i}\rightarrow X_{j} \mid \gD)\big)^{2}\right)^{1/2}.
\end{equation*}
In \cref{tab:dag-gfn-comparison-graphs}, we also report the RMSE for other marginals: the marginal of having a directed path between two nodes $P(X_{i} \rightsquigarrow X_{j}\mid \gD)$, as well as the marginal of node $X_{i}$ being in the Markov blanket of $X_{j}$ $P(X_{i}\in \mathrm{MarkovBlanket}(X_{j})\mid \gD)$. In addition to the RMSE between those marginals, we also report the Pearson correlation coefficient, as in \citep{deleu2022bayesian}. Note that no metric is reported on graphs over $d=3$ nodes for BCD Nets \citep{cundy2021bcdnets} due to technical reasons (the method is not applicable for graphs smaller than $4$ nodes).

\begin{table}[t]
    \centering
    \caption{Comparison between GFlowNet and other methods based on variational inference on the Bayesian structure learning task, for different marginals of interest of the distribution over graphs $P(G\mid \gD)$.}
    \vspace{0.5ex}
    \label{tab:dag-gfn-comparison-graphs}
    \begin{tabular}{llcccccc}
        \toprule
        & & \multicolumn{3}{c}{RMSE} & \multicolumn{3}{c}{Pearson's r}\\\cmidrule(lr){3-5}\cmidrule(lr){6-8}
        \multicolumn{2}{c}{Number of variables ($d$)} & 3 & 4 & 5 & 3 & 4 & 5 \\
        \midrule
        \multirow{3}{*}{Edges} & BCD Nets & -- & $2.13 \times 10^{-1}$ & $2.61 \times 10^{-1}$ & -- & $0.8578$ & $0.7886$ \\
        & DiBS & $3.28 \times 10^{-1}$ & $2.95 \times 10^{-1}$ & $3.15 \times 10^{-1}$ & $0.6903$ & $0.7085$ & $0.7170$ \\
        & GFlowNet & $\mathbf{1.50 \times 10^{-2}}$ & $\mathbf{1.61 \times 10^{-2}}$ & $\mathbf{1.80 \times 10^{-2}}$ & $\mathbf{0.9993}$ & $\mathbf{0.9990}$ & $\mathbf{0.9990}$ \\
        \midrule
        \multirow{3}{*}{Paths} & BCD Nets & -- & $2.59 \times 10^{-1}$ & $3.08 \times 10^{-1}$ & -- & $0.8378$ & $0.7500$ \\
        & DiBS & $3.50 \times 10^{-1}$ & $3.35 \times 10^{-1}$ & $3.48 \times 10^{-1}$ & $0.6951$ & $0.7080$ & $0.7020$ \\
        & GFlowNet & $\mathbf{3.39 \times 10^{-3}}$ & $\mathbf{1.07 \times 10^{-2}}$ & $\mathbf{1.99 \times 10^{-2}}$ & $\mathbf{1.0000}$ & $\mathbf{0.9996}$ & $\mathbf{0.9989}$ \\
        \midrule
        \multirow{3}{*}{\shortstack[l]{Markov\\blanket}} & BCD Nets & -- & $3.02 \times 10^{-1}$ & $3.49 \times 10^{-1}$ & -- & $0.8831$ & $0.7864$ \\
        & DiBS & $3.88 \times 10^{-1}$ & $3.80 \times 10^{-1}$ & $4.45 \times 10^{-1}$ & $0.7840$ & $0.7892$ & $0.6888$ \\
        & GFlowNet & $\mathbf{2.14 \times 10^{-2}}$ & $\mathbf{2.38 \times 10^{-2}}$ & $\mathbf{2.83 \times 10^{-2}}$ & $\mathbf{0.9986}$ & $\mathbf{0.9982}$ & $\mathbf{0.9980}$ \\
        \bottomrule
    \end{tabular}
\end{table}

To evaluate the accuracy of the approximation on the continuous part of the distribution, we report in \cref{tab:dag-gfn-comparison} the (average) negative log-probability of the sampled parameters $\theta_{k}$ from the different approximations (GFlowNet, DiBS \citep{lorch2021dibs}, and BCD Nets \citep{cundy2021bcdnets}) against the exact posterior distribution $P(\theta \mid G, \gD)$. More precisely, we compute
\begin{equation*}
    \mathrm{Measure}_{\theta}(\hat{P}, P) = -\frac{1}{K|\gD|}\sum_{k=1}^{K}\log P(\theta_{k}\mid G_{k}, \gD)
\end{equation*}
In other words, the lower this metric is, the more likely the samples $\theta_{k}$ from those approximations are under the exact posterior distribution over parameters.

\subsection{Connections with diffusion models}
\label{appendix:ddpm}

\begin{figure}
\centering
\includegraphics[width=.48\linewidth]{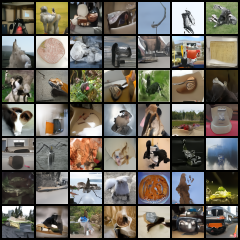}
\includegraphics[width=.48\linewidth]{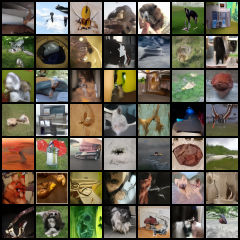}
\caption{Generated samples from MLE-GFN on ImageNet-$32$ dataset.}
\label{fig:ddpm}
\end{figure}
We train a diffusion model-specified GFlowNet with $T=100$ for $200,000$ steps. This is much shorter than other state-of-the-art work (such as \citet{Lipman2022FlowMF}) and takes less than $3$ days on a single V100 GPU. All NLL results are computed in bits per dimension (BPD). We use $50,000$ generated samples to compute the FID score for evaluating the sample quality. The Adam learning rate is $2\times 10^{-4}$ for the forward policy and $2\times 10^{-5}$ for the backward policy. The parameter of backward policy is $\{\phi_i\}_{i=1}^T$, where the variance coefficient $\beta_i$ satisfies $\beta_i = \bar\beta_i\cdot\exp{(\phi_i)}$ and $\bar\beta_i$ is the original variance coefficient used in \citet{ho2020ddpm}. For details about the MLE-GFN algorithm, we refer to \citet{zhang2023unifying}.
\cref{fig:ddpm} shows examples of images generated by the algorithm.

\section{Additional lemmas and propositions}
In this section, we will write and prove lemmas and propositions that would help explain and shorten the proofs of the main results presented in the main text and in \cref{appendix:transition-kernels}, which we provide in \cref{appendix:proofs}.

The following lemma ensures that in a measurable pointed graph, $\{s_0\}$ is not accessible.
\begin{lemma}
$\forall s \in \gS, \ \kappa(s, \{s_0\})=0$
\label{lemma:no-incoming-so-edge-appendix}
\end{lemma}

\begin{proof}
We will present a proof by contradiction.

Let $\gN = \{ s \in \bar{\gS}, \ \kappa(s, \{s_0\}) > 0\}$, and assume that $\gN \neq \emptyset$. $(0, \infty)$ being an open set, and $s \mapsto \kappa(s, \{s_0\})$ continuous, this means that $\gN \in \bar{\gT}$ (i.e. it is open). From \cref{eq:accessibility}, it follows that there is some $n \geq 0$ such that $\kappa^n(s_0, \gN) >0$.

Applying \cref{eq:n-step-measure}, we obtain:
\begin{align*}
    \kappa^{n+1}(s_0, \{s_0\}) = \int_{\bar{\gS}} \kappa^n(s_0, ds') \kappa(s', \{s_0\}) 
    \geq \int_{\gN} \kappa^n(s_0, ds') \kappa(s', \{s_0\})
    >0.
\end{align*}
Writing, for all $m > 1$:
\begin{align*}
    \kappa^{m(n+1)}(s_0, \{s_0\}) &= \int_{\bar{\gS}} \kappa^{(m-1)(n+1)}(s_0, ds') \kappa^{n+1}(s', \{s_0\}) 
    \geq \int_{\{s_0\}} \kappa^{(m-1)(n+1)}(s_0, ds') \kappa^{n+1}(s', \{s_0\}) \\
    &= \kappa^{(m-1)(n+1)}(s_0, \{s_0\}) \kappa^{n+1}(s_0, \{s_0\}),
\end{align*}
it follows from a simple induction that $\forall m \geq 1$, $\kappa^{m(n+1)}(s_0, \{s_0\}) > 0$, which contradicts \cref{eq:finitely-absorbing}.

$\gN$ is thus necessarily empty.
\end{proof}

The following lemma ensures a compatibility between the definition of the terminating states $\gX$ and $\kappa^b(\bot, -)$:
\begin{lemma}
The support of $\kappa^b(\bot, -)$ is the closure of $\gX$.
\end{lemma}
\begin{proof}
Let $s$ be an element in the support of $\kappa^b(\bot, -)$. By definition of the support, it means that for any $B \in \gT$ containing $s$, there is some $B' \subseteq B$ such that $B' \subseteq \gX$. In particular, $B \cap \gX \neq \emptyset$. This means that $s$ is a point of closure of $\gX$.

Conversely, let $s$ be a point of closure of $\gX$. Given any open set $B \in \gT$ containing $s$, $s$ being a closure point means that $B \cap \gX$ (which is measurable) is non-empty. Following \cref{eq:backward-kernel}, we get:
\begin{align*}
    \nu(\{\bot\}) \kappa^b(\bot, B\cap\gX) = \int_{\bar{\gS}} \indicator_{B \cap \gX}(s') \nu(ds') \kappa(s', \{\bot\})
\end{align*}
The RHS of the previous equality is positive because $\nu$ is a strictly positive measure and $\kappa(s', \{\bot\}) > 0$ for every $s' \in \gX$, following \cref{def:terminating-states}. Hence $\kappa^b(\bot, B\cap\gX) > 0$. It follows that $\kappa^b(\bot, B) \geq \kappa^b(\bot, B\cap\gX) > 0$. Meaning that $s$ is indeed within the support of $\kappa^b(\bot, -)$.
\end{proof}

The proof of \cref{lemma:terminating-state-dist-PF} relies on an the following intermediary lemma, which relates the $n$-step trajectory measures $P_F^{\otimes{n}}(s, -)$ to the $n$-step measures $P_F^n(s, -)$ defined by \cref{eq:n-step-measure}.
\begin{lemma}
    Let $P_F$ be a transition kernel on $(\bar{\gS}, \Sigma)$. For every $s \in \bar{\gS}, \ n \geq 0$, and for any bounded measurable function $f: \bar{\gS} \to \R$, we have:
    \begin{equation}
        \label{eq:P_F^n}
        \int_{\bar{\gS}^{n+1}} f(s') P_F^{\otimes{n}}(s, ds_1\dots ds_n ds') = \int_{\bar{\gS}} f(s') P_F^n(s, ds')
    \end{equation}
    \label{lemma:P_F^n}
\end{lemma}
\begin{proof}
    We prove the lemma by induction on $n$. First, for $n=0$, using \cref{eq:n-step-trajectory-measure-initial-condition}
    \begin{equation*}
        \int_{\bar{\gS}} f(s') P_F^{\otimes{0}}(s, ds') = f(s) = \int_{\bar{\gS}} f(s') P_F^0(s, ds').
    \end{equation*}
    Then, assuming that \cref{eq:P_F^n} is satisfied for some $n \geq 0$, we get:
    \begin{align*}
        \int_{\bar{\gS}^{n+2}} f(s') P_F^{\otimes{n+1}}(s, ds_1\dots ds_{n+1} ds')
        &=\int_{\bar{\gS}^{n+2}} f(s') P_F^{\otimes{n}}(s, ds_1\dots ds_{n+1} )P_F(s_{n+1}, ds') \\
        &=\int_{\bar{\gS}^{n+1}} \underbrace{\int_{\bar{\gS}} f(s') P_F(s_{n+1}, ds')}_{\triangleq g(s_{n+1})}P_F^{\otimes{n}}(s, ds_1\dots ds_{n+1} )\\
        &=\int_{\bar{\gS}} g(s_{n+1}) P_F^n(s, ds_{n+1})
        =\iint_{\bar{\gS}\times\bar{\gS}} f(s') P_F(s_{n+1}, ds')P_F^n(s, ds_{n+1})\\
        &=\int_{\bar{\gS}} f(s') P_F^{n+1}(s, ds_{n+1}),
    \end{align*}
    where we applied the inductive hypothesis to a new bounded and measurable\footnote{This can be seen by writing $f=f^+ - f^-$, where $f^+, f^-$ are non-negative, writing each of $f^+, f^-$ as a limit of step functions, and using the monotone convergence theorem.} function $g$, and applied the recursive definition of $P_F^{n+1}$.
\end{proof}

The next proposition is crucial in proving \cref{thm:fm-implies-correct-sampling}.
\begin{proposition}
    Let $F=(\mu, P_F)$ be a flow over $G$ (i.e. $F$ satisfies the flow-matching conditions \cref{eq:flow-matching-conditions}), then the measure $u$ defined by:
    \begin{equation}
        \label{eq:incomplete-trajectory-termination}
        u: B \in \Sigma_{| \gS} \mapsto \sum_{n=0}^{\infty} P_F^n(s_0, B),
    \end{equation}
    is finite, and satisfies for all $B \in \Sigma_{| \gS}$
    \begin{equation}
        \label{eq:incomplete-trajectory-termination-Z}
        \mu(\{s_0\}) u(B) = \mu(B)
    \end{equation}
    \label{prop:incomplete-trajectory-termination}
\end{proposition}
\begin{proof}
    First, using a simple recursion, we show that $\forall n \geq N, \ P_F^n(s_0, - ) = \delta_\bot$. The base case ($n=N$) is satisfied as a consequence of \cref{lemma:absolute-continuity-powers}, and the fact that the measurable pointed graph is finitely absorbing. Assuming it holds for some $n \geq N$, going back to the definition of the $n$-step measure, we have for every $B \in \Sigma$:
    \begin{align*}
        P_F^{n+1}(s_0, B) &= \int_{\bar{\gS}} P_F^n(s_0, ds) P_F(s, B) 
        = \int_{\bar{\gS}} \delta_\bot(ds) P_F(s, B)
        = P_F(\bot, B)
        =\delta_\bot(B),
    \end{align*}
    where the last equality stems from the absolute continuity of $P_F(\bot, -)$ wrt. $\kappa(\bot, -)$ and \cref{eq:finitely-absorbing}.

    This shows that for every $B \in \Sigma_{| \gS}$:
  \begin{equation*}
        u(B) = \sum_{n=0}^{N-1} P_F^n(s_0, B).
    \end{equation*}
    Which shows the measure $u$ is finite.

    Next, we partition $\gS$ into $N$ disjoint sets $\gS_0, \dots, \gS_{N-1}$, where:
    \begin{equation*}
        s \in \gS_n \Leftrightarrow n = \max \{m \in \sN_0: \ \forall B \in \gT, s \in B \Rightarrow P_F^m(s_0, B) > 0\}
    \end{equation*}
    $\gS_n \in \Sigma$ given that $\gS_n = \gS'_n \setminus \bigcup_{k=1}^\infty \gS'_{n+k}$, where $\gS'_n$ is the support of $P_F^n(s_0, -)$, which is known to be a closed set, and hence measurable.
    
    Writing any $B \in \Sigma_{| \gS}$ as:
    \begin{equation*}
        B = \bigcup_{n=0}^{N-1} B \cap \gS_n,
    \end{equation*}
    and using the additivity property of the measures $u$ and $\mu$, then proving \cref{eq:incomplete-trajectory-termination-Z} for all $B \in \Sigma_{| \gS}$, amounts to proving it for all $B \in \Sigma_{| \gS_n}$ for all $n \in \{0, \dots, N-1\}$, given that the sets $\gS_i$ are themselves measurable. We prove this by strong induction on $n$.

    \emph{Base case:} For $n=0$, $\gS_0=\{s_0\}$, and $\Sigma_{| \gS_0} = \{\{s_0\}\}$. 

    \begin{equation*}
        u(\{s_0\}) = P_F^0(s_0, \{s_0\}) = \delta_{s_0}(\{s_0\}) = 1,
    \end{equation*}
    Hence \cref{eq:incomplete-trajectory-termination-Z} is satisfied for $B=\{s_0\}$.

    \emph{Induction step:} Assume that for some $n \geq 0$, \cref{eq:incomplete-trajectory-termination-Z} is satisfied for all $B \in \Sigma_{| \gS_m}$ for all $m \leq n$, and let $B \in \Sigma_{| \gS_{n+1}}$.

    Define $B' = \{ s' \in \gS \ : \ P_F(s', B) > 0 \}$. We first show that $B' \subseteq \bigcup_{m=0}^{n} \gS_{m}$. If there is $s' \in B'$ and $n' > n$ such that $s' \in \gS_{n'}$, then for any open set $\tilde{B}$ (i.e. $\tilde{B} \in \gT$) containing $s'$, $P_F^{n'}(s_0, \tilde{B}) > 0$. $\tilde{s} \mapsto P_F(\tilde{s}, B)$ being a continuous function, it follows that $B'$ itself is open. Hence $P_F^{n'}(s_0, B')>0$. And noting that:
    \begin{align*}
        P_F^{n'+1}(s_0, B) &= \int_{\bar{\gS}} P_F^{n'}(s_0, ds') P_F(s', B) 
         \geq \int_{B'} P_F^{n'}(s_0, ds') P_F(s', B)  > 0,
    \end{align*}
    which would contradict the fact that $B \subseteq \gS_{n+1}$, given that $n'+1 > n+1$. This shows that $B' \subseteq \bigcup_{m=0}^n \gS_m$. 

    Hence:
    \begin{align*}
        \mu(\{s_0\})u(B) &= \mu(\{s_0\})\sum_{m=0}^{n+1}P_F^m(s_0, B)
        = \mu(\{s_0\})\underbrace{\delta_{s_0}(B)}_{=0} + \mu(\{s_0\})\sum_{m=0}^n P_F^{m+1}(s_0, B)\\
        &= \mu(\{s_0\})\sum_{m=0}^n \int_{B'} P_F^m(s_0, ds') P_F(s', B) 
        = \int_{B'} \mu(\{s_0\})u(ds') P_F(s', B) \\
        &= \int_{B'} \mu(ds') P_F(s', B) 
        = \mu(B),
    \end{align*}
    where the last two equalities stem from the induction hypothesis and the flow matching conditions respectively. 
\end{proof}

The following lemma,  relates the flow measures at the source and sink states to the reward measure and shows that the source and sink flows correspond to the "total reward" $R(\gX)$ (called the partition function with discrete GFlowNets):
\begin{lemma}
    Let $F=(\mu, P_F)$ be a flow over $G$ satisfying reward-matching conditions in \cref{eq:boundary-condition} w.r.t.\ a measure $R$, then
    \begin{equation}
        \mu(\{s_0\})=\mu(\{\bot\})=R(\gX).
        \label{eq:z-equals-total-reward}
    \end{equation}
    \label{lem:z-equals-total-reward}
    \vspace{-1em}
\end{lemma}

\begin{proof}
    Applying \cref{eq:boundary-condition} to the function $f:x \in \gX \mapsto 1$, we get:
    \begin{align}
        \int_\gX R(dx) = \int_\gX \mu(dx) P_F(x, \{\bot\}).
    \end{align}
    Additionally, from \cref{eq:terminating-states}, we get $\forall s \in \gS \setminus \gX, \ \kappa(s, \{\bot\}) = 0$. It follows from the absolute continuity requirements of $P_F$ that $\forall s \in \gS \setminus \gX, \ P_F(s, \{\bot\}) = 0$. Hence:
    \begin{align*}
        \int_\gX R(dx) &= \int_\gS \mu(ds) P_F(s, \{\bot\})
        = \iint_{\gS \times \bar{\gS}} \indicator_{s' = \bot} \mu(ds) P_F(s, ds')
        = \int_{\bar{\gS}} \indicator_{s' = \bot} \mu(ds') = \mu(\{\bot\}),
    \end{align*}
    where the last line follows from \cref{eq:flow-matching-conditions}. This shows that $\mu(\{\bot\})=R(\gX)$.

    Note that as a consequence of \cref{lemma:no-incoming-so-edge-appendix} and the absolute continuity requirement, a $P_F$ satisfies:
    \begin{equation}
        \forall s \in \bar{\gS}, P_F(s, \{s_0\}) = 0
        \label{eq:no-incoming-s0-edge-PF}
    \end{equation}

    Next, following \cref{eq:flow-matching-conditions} and \cref{eq:no-incoming-s0-edge-PF}, we have:
    \begin{align}
    \iint_{\gS\times\bar{\gS}}\mu(ds)P_F(s,ds')
    &=\iint_{\gS\times\bar{\gS}}\indicator_{s'\neq s_0}\mu(ds)P_F(s,ds')
    =\int_{\bar{\gS}}\indicator_{s'\neq s_0}\mu(ds')
    =\mu(\bar{\gS})-\mu(\{s_0\}).\label{eq:measure-full-left}
    \end{align}
    On the other hand, because each $P_F(s, -)$ is a probability measure on $\bar{\gS}$:
    \begin{align}
    \iint_{\gS\times\bar{\gS}}\mu(ds)P_F(s,ds')
    &=\int_{\gS}\left(\int_{\bar{\gS}}P_F(s,ds')\right)\mu(ds)=\int_{\gS}\mu(ds) = \mu(\gS)
    \label{eq:measure-full-right}
    \end{align}
    Subtracting \cref{eq:measure-full-left} from \cref{eq:measure-full-right}, we get:
    \begin{equation*}
    \mu(\{s_0\})=\mu(\bar{\gS}) - \mu(\gS) = \mu(\{\bot\}) = R(\gX)
    \end{equation*}
\end{proof}

The following lemma is crucial to prove \cref{prop:TB-implies-FM-RM}.
\begin{lemma}
    If $(P_F, P_B, Z)$ satisfy the trajectory balance conditions wrt. $R$, then for any  $n \in \{0,...,N\}$, and for any measurable bounded function $f: \gS \rightarrow \sR$:
    \begin{equation}
        Z\int_{\gS} f(s)P_F^n(s_0,ds)P_F(s,\{\bot\}) = \int_{\gS} f(s)P_B^{n}(s,\{s_0\})R(ds)
    \end{equation}
    \label{lemma:integrating-PF-PB}
\end{lemma}

\begin{proof}
Using \cref{lemma:P_F^n}, we have:
\begin{align*}
    Z\int_{\gS} f(s)P_F^n(s_0,ds)P_F(s,\{\bot\}) &= Z \int_{\bar{\gS}^{n+1}} \indicator_{s \neq \bot}f(s) P_F(s, \{\bot\}) P_F^{\otimes n}(s_0, ds'ds_1\dots ds_{n-1}ds)\\
    &=Z \int_{\bar{\gS}^{n+2}} \indicator_{s \neq \bot, s_{n+1}=\bot} f(s) P_F^{\otimes n+1}(s_0, ds'ds_1\dots ds_{n-1}ds ds_{n+1}) \\
    &=\int_{\bar{\gS}^{n+2}} \indicator_{s \neq \bot}\indicator_{s''=s_0}f(s)R(ds)P_B^{\otimes n}(s, ds'ds_{n-1}\dots ds_1, ds'')\\
    &=\int_{\bar{\gS}} f(s)\indicator_{s \neq \bot} R(ds) \int_{\bar{\gS}^{n+1}} \indicator_{s''=s_0} P_B^{\otimes n}(s, ds'ds_{n-1}\dots ds_1, ds'')\\
    &=\int_{\bar{\gS}} f(s)\indicator_{s \neq \bot} R(ds) \int_{\bar{\gS}} \indicator_{s''=s_0} P_B^n(s, ds'')\\
    &=\int_{\gS} f(s) R(ds) P_B^n(s, \{s_0\})
\end{align*}
\end{proof}

The following proposition generalizes Lemma 5 of ~\citep{bengio2021foundations} to measurable pointed graphs, and is crucial in proving \cref{prop:TB-implies-FM-RM}
\begin{proposition}
        Let $P_B$ be a backward kernel over $G$. Let $P_{B, T}$ be the measure defined by:
    \begin{equation}
        P_{B, T}(s) = \sum_{n = 0}^{\infty}P_B^n(s, \{s_0\})
    \end{equation}
    We have $\forall s \in \gS$:
    \begin{equation}
        P_{B, T}(s) = 1
        \label{eq:p-b-total}
    \end{equation}
    \label{prop:p-b-total}
\end{proposition}

\begin{proof}
    First, using a simple recursion, we show that $\forall n \geq N, \ P_B^n(s, \{s_0\} ) = 0$. The base case ($n=N$) is trivially satisfied because the measurable pointed graph has maximal trajectory length $N$ and $P_B(s_0, -)$ is the trivial measure (i.e. it assigns zero to every measurable set), given that it is absolutely continuous wrt. $\kappa^b(s_0, -)$. Assuming it holds for some $n \geq N$, we have:
    \begin{align*}
        P_B^{n+1}(s, {s_0}) &= \int_{\bar{\gS}} P_B(s, ds') P_B^{n}(s', \{s_0\}) 
        = \int_{\bar{\gS}} P_B(s, ds').0 
        = 0
    \end{align*}

    This shows that for every $s \in \gS$:
  \begin{equation*}
        P_{B, T}(s) = \sum_{n=0}^{N-1} P_B^n(s, {ds_0}).
    \end{equation*}
    Which shows the measure $P_{B, T}$ is finite.

    Next, we partition $\gS$ into $N$ disjoint sets $\gS_0, \dots, \gS_{N-1}$, where:
    \begin{equation*}
        s \in \gS_n \Leftrightarrow n = \max \{m \in \sN_0: P_B^m(s, \{s_0\}) > 0\}
    \end{equation*}
    $\gS_n \in \Sigma$ given that $\gS_n = \gS'_n \setminus \bigcup_{k=0}^\infty \gS'_{n+k}$, where $\gS'_n$ is the support of $P_B^n(-, \{s_0\})$, which is known to be a closed set, and hence measurable.
    
    Writing :
    \begin{equation*}
        \gS = \bigcup_{n=0}^{N-1}  \gS_n,
    \end{equation*}
    Then proving \cref{eq:p-b-total} for all $s \in \gS$, amounts to proving it for all $s \in \gS_n$ for all $n \in \{0, \dots, N-1\}$. We prove this by strong induction on $n$.

    \emph{Base case:} For $n=0$, $\gS_0=\{s_0\}$, and 
    \begin{equation*}
        P_{B, T}(s_0) = P_B^0(s_0, \{s_0\}) = \delta_{s_0}(\{s_0\}) = 1,
    \end{equation*}
    Hence it is satisfied for n = 0.

    \emph{Induction step:} Assume that for some $n \geq 0$, \cref{eq:p-b-total} is satisfied for all $s \in  \gS_m$ for all $m \leq n$, and let $s \in \gS_{n+1}$.

    Define $B_s = \{ s' \in \gS   :  \forall B \in \gT, s' \in B => P_B(s, B) > 0 \}$. We first show by contradiction that $B_s \subseteq  \bigcup_{m=0}^{n} \gS_{m}$. If there is $s' \in B_s$ and $n' > n$ such that $s' \in \gS_{n'}$, then $P_B^{n'}(s', \{s_0\}) > 0$, and by continuity of $\tilde{s} \mapsto P_B^{n'}(\tilde{s}, \{s_0\})$, there exists an open set $\tilde{B}$ (i.e. $\tilde{B} \in \gT$) containing $s'$ such that $P_B^{n'}(\tilde{s}, \{s_0\}) > 0$ for all $\tilde{s} \in \tilde{B}$. Hence:
    \begin{align*}
        P_B^{n'+1}(s, \{s_0\}) = \int_{\bar{\gS}} P_B(s, ds') P_B^{n'}(s', \{s_0\}) 
         \geq \int_{\tilde{B}} P_B(s, ds') P_B^{n'}(s', \{s_0\})  > 0,
    \end{align*}
    which would contradict the fact that $s \in \gS_{n+1}$, given that $n'+1 > n+1$.

    Hence:
    \begin{align*}
        & P_{B, T}(s) = \sum_{m = 0}^{n+1}P_B^m(s, \{s_0\}) 
        = \sum_{m = 1}^{n+1}P_B^m(s, \{s_0\}) \quad \text{given that $s \neq s_0$} \\
        &= \sum_{m=1}^{n+1} \int_{s' \in B_s} P_B(s,ds')P_B^{m-1}(s', \{s_0\}) 
        = \int_{s' \in B_s} P_B(s,ds') \underbrace{\sum_{m=0}^{n}P_B^m(s',\{s_0\})}_{=1} 
        = \int_{s' \in B_s} P_B(s,ds') 
        = 1
    \end{align*}

\end{proof}

\section{Proofs of results in the main text}
\label{appendix:proofs}
We are now ready to prove the main theorem of the paper, which we restate here:
\begin{customthm}{\ref{thm:fm-implies-correct-sampling}}
    If $F=(\mu, P_F)$ is a flow over $G$, that satisfies the reward matching conditions \cref{eq:boundary-condition} wrt. a measure $R$, then the terminating state measure:
    \begin{equation}
        P_T: B \in \Sigma_{| \gX} \mapsto \sum_{n=1}^\infty P_T^n(B)
    \end{equation}
    is a probability measure and satisfies for all $B \in \Sigma_{| \gX}$:
    \begin{equation}
        P_T(B) = \frac{1}{R(\gX)} R(B)
    \end{equation}
\end{customthm}
\begin{proof}
    Using \cref{lemma:terminating-state-dist-PF}, the terminating state measure $P_T$ satisfies for any bounded measurable function $f: \gX \ra \R$:
    \begin{align*}
        \int_{\gX} f(x) P_T(dx) = \int_{\gX} f(x) P_F(x, {\bot}) \sum_{n=0}^\infty P_F^n(s_0, dx).
    \end{align*}
    It follows from \cref{prop:incomplete-trajectory-termination} that 
    \begin{align*}
        \mu(\{s_0\}) \int_{\gX} f(x) P_T(dx) = \int_{\gX} f(x) P_F(x, {\bot}) \mu(dx).
    \end{align*}
    Following \cref{lem:z-equals-total-reward}, and the positivity assumption on $R$ that:
    \begin{align*}
        \int_{\gX} f(x) P_T(dx) = \frac{1}{R(\gX)}\int_{\gX} f(x) P_F(x, {\bot}) \mu(dx).
    \end{align*}
    Finally, using \cref{eq:boundary-condition}, we obtain:
    \begin{align*}
        \int_{\gX} f(x) P_T(dx) = \frac{1}{R(\gX)}\int_{\gX} f(x) R(dx).
    \end{align*}
    $P_T$ being a probability measure follows by applying the last equality to the function $f: x \mapsto 1$. 
\end{proof}

Next, we will prove \cref{prop:DB-implies-FM}.
\begin{customproposition}{\ref{prop:DB-implies-FM}}
    If $(\mu, P_{F}, P_{B})$ satisfy the detailed balance conditions in \cref{def:db-conditions}, then $F = (\mu, P_{F})$ satisfies the flow-matching conditions in \cref{def:flow-and-flow-matching} and is thus a flow.
\end{customproposition}

\begin{proof}
    For any bounded measurable function $f: \bar{\gS} \rightarrow \sR$ satisfying $f(s_0)=0$, we can define a function $g: \gS \times \bar{\gS} \rightarrow \sR$ such that for all $(s, s') \in \gS \times \bar{\gS}$, $g(s, s') = f(s')$. Note that $g$ satisfies $g(s, s_0)=0$ for every $s \in \gS$. Applying the detailed balance conditions to the function $g$, we have
    \begin{equation*}
        \iint_{\gS \times \bar{\gS}} g(s, s')\mu(ds)P_{F}(s, ds') = \iint_{\gS \times \bar{\gS}} f(s')\mu(ds)P_{F}(s, ds')
    \end{equation*}
    On the other hand, using the RHS of \cref{eq:db-conditions} in the detailed balance conditions, we get
    \begin{align*}
        \iint_{\gS \times \bar{\gS}}& g(s, s')\mu(ds')P_{B}(s', ds)= \iint_{\gS \times \bar{\gS}} f(s')\mu(ds')P_{B}(s', ds)
        = \int_{\bar{\gS}} f(s')\mu(ds') \int_{\gS} P_B(s', ds),
    \end{align*}
    Following \cref{eq:absorption} and the the absolute continuity conditions of $P_B$ with respect to $\kappa^b$, we have:
    \begin{align*}
        \forall s' \in \gS, \ P_B(s', \{\bot\}) = 0,
    \end{align*}
    from which it follows that:
    \begin{align*}
        \iint_{\gS \times \bar{\gS}}& g(s, s')\mu(ds')P_{B}(s', ds)=\int_{\bar{\gS}} f(s')\mu(ds') \underbrace{\int_{\bar{\gS}} P_B(s', ds)}_{=1}
    \end{align*}
   This shows that $(\mu, P_F)$ satisfy the flow matching conditions.
\end{proof}

Next, we will prove \cref{prop:TB-implies-FM-RM}, which we restate here:

\begin{customproposition}{\ref{prop:TB-implies-FM-RM}}
    If $(Z, P_F, P_B)$ satisfy the trajectory balance condition \cref{eq:tb-conditions} wrt. a measure $R$, then $F=(\mu, P_B)$, where $\mu$ is defined by
    \begin{enumerate}[label=(\arabic*),left=0pt,nosep,]
    \item $\mu(\{\bot\})=\mu(\{s_0\}) = Z$
    \item $\forall B \in \Sigma_{| \gS}$:
    $\mu(B) = \mu(\{s_0\})\sum_{n=0}^{\infty} P_F^n(s_0, B) $
    \end{enumerate}
    satisfies both the flow-matching conditions \cref{eq:flow-matching-conditions} and the reward matching conditions \cref{eq:boundary-condition} wrt. $R$.
\end{customproposition}

\begin{proof}
First we show that $\mu$ satisfies the flow-matching condition.
for any bounded measurable function $f: \bar{\gS} \rightarrow \sR$ satisfying $f(s_0)=0$, we have : 
\begin{align*}
     \iint_{\gS \times \bar{\gS}} f(s')\mu(ds)P_{F}(s, ds')
    &= \iint_{\gS \times \bar{\gS}} f(s')\mu(s_0)\sum_{n=0}^{\infty} P_F^n(s_0, ds)P_{F}(s, ds')\\
    &= \int_{ \bar{\gS}} f(s')\mu(s_0)\sum_{n=0}^{\infty} P_F^{n+1}(s_0, ds')\\
    &= \int_{ \bar{\gS}} f(s')\mu(s_0)\sum_{n=0}^{\infty} P_F^{n}(s_0, ds')\quad \text{ (because $f(s_0) = 0$)}\\
    &= \int_{\bar{\gS}} f(s')\mu(ds')
\end{align*}

Now, we will show the reward matching condition.
For any bounded measurable function $f:  \gX \rightarrow \sR$
\begin{align*}
\int_{\gX}f(s)\mu(ds)P_F(s,\bot)
    &=\int_{\gS}\indicator_\gX(s)f(s) \underbrace{\mu(\{s_0\})}_{=Z}\sum_{n=0}^{\infty}P_F^{n}(s_0,ds)P_F(s,{\bot}) \\
    &=\sum_{n=0}^{\infty} \int_{\gS} \indicator_\gX(s)f(s)R(ds)P_B^n(s,\{s_0\})\\
    &= \int_{\gS} \indicator_\gX(s)f(s)R(ds) \sum_{n=0}^{\infty}P_B^n(s,\{s_0\})
    = \int_{\gX} f(s)R(ds)
\end{align*}

where we used \cref{prop:p-b-total} and \cref{lemma:integrating-PF-PB}.

\end{proof}

The following is the proof of \cref{thm:zero-losses-implies-conditions}, which we restate here:
\begin{customthm}{\ref{thm:zero-losses-implies-conditions}}
\begin{enumerate}[left=0pt,nosep,label=(\arabic*),wide]
    \item If $L_{FM}(-;\theta) = 0$ $\nu$-almost surely, then $F=(\mu, P_{F})$ is a flow (i.e. satisfies the flow-matching conditions in \cref{def:flow-and-flow-matching}).
    \item If $L_{DB}(-;\theta) = 0$ $\nu \otimes \kappa$-almost surely, then $(\mu, P_{F}, P_{B})$ satisfy the detailed balance conditions in \cref{def:db-conditions}.
    \item If $L_{RM}(-;\theta)=0 $ $\nu_{|\gX}$-almost surely, then $(\mu, P_F)$ satisfies the reward matching conditions in \cref{eq:boundary-condition}.
    \item If $L_{TB}^n(-; \theta)=0$ $((\nu \otimes \kappa^{\otimes n+1})_{|\{s_0\}\times \gS^n \times \{\bot\}}$)-almost surely for every $n \geq 0$, then $(Z\nu(\{s_0\}), P_F, P_B)$ satisfy the trajectory balance condition in \cref{def:tb-conditions}.
\end{enumerate}
\end{customthm}
\begin{proof}
    We will first show the result for the flow-matching loss. Let the function $v: \bar{\gS} \ra \sR_{+}$ be the function, depending on the parameter $\theta$, defined by $\forall s'\in\bar{\gS}$:
    \begin{equation*}
        v(s';\theta) := \int_{\gS} u(s;\theta)p_{F}(s, s';\theta)\kappa^{b}(s', ds).
    \end{equation*}
    If we assume that $L_{FM}(-;\theta) = 0$ $\nu$-almost surely, then we have equivalently $u(-;\theta) = v(-;\theta)$ $\nu$-almost surely. Let $f: \bar{\gS} \ra \sR$ be a bounded measurable function such that $f(s_{0}) = 0$, we then have
    \begin{align*}
        \int_{\bar{\gS}}& f(s')u(s';\theta) \nu(ds') = \int_{\bar{\gS}}f(s')v(s';\theta)\nu(ds')\\
        &= \int_{\bar{\gS}}f(s')\int_{\gS}u(s;\theta)p_{F}(s, s';\theta)\kappa^{b}(s', ds)\nu(ds')\\
        &= \iint_{\bar{\gS}\times \gS}f(s')u(s;\theta)p_{F}(s, s';\theta)\kappa(s, ds')\nu(ds),
    \end{align*}
    where we used the fact that $\nu \otimes \kappa = \nu \otimes \kappa^{b}$ (from \cref{eq:backward-kernel} in \cref{def:measurable_pointed_graph}) in the last equality. Replacing the densities (and reference measures) in the equality above with their corresponding measures $\mu$ and $P_{F}$, we get
    \begin{equation*}
        \int_{\bar{\gS}} f(s')\mu(ds') = \iint_{\gS\times\bar{\gS}} f(s')\mu(ds)P_{F}(s, ds').
    \end{equation*}
    Since this equality is valid for any bounded measurable function $f$ satisfying $f(s_{0}) = 0$, this is the definition of $F = (\mu, P_{F})$ satisfying the flow-matching conditions (\cref{def:flow-and-flow-matching}).

    The proof for the detailed balance loss is similar. Let the functions $g: \bar{\gS} \times \bar{\gS} \ra \sR_{+}$ and $h: \bar{\gS} \times \bar{\gS} \ra \sR_{+}$ defined as
    \begin{align*}
        g(s, s';\theta) &:= u(s;\theta)p_{F}(s,s';\theta)\\
        h(s, s';\theta) &:= u(s';\theta)p_{B}(s',s;\theta).
    \end{align*}
    If $L_{DB}(-;\theta) = 0$ $\nu\otimes \kappa$-almost surely, then we have equivalently $g(-;\theta) = h(-,\theta)$. Let $f: \gS \times \bar{\gS} \ra \sR$ be a bounded measurable function such that $f(s, s_{0}) = 0$ for all $s\in \gS$. We have
    \begin{align*}
        \iint_{\gS \times \bar{\gS}} & f(s, s')g(s, s';\theta) (\nu \otimes \kappa)(ds, ds') \\
        &= \iint_{\gS \times \bar{\gS}} f(s, s')u(s;\theta)p_{F}(s, s';\theta)(\nu \otimes\kappa)(ds\, ds')\\
        &= \iint_{\gS \times \bar{\gS}} f(s, s')h(s, s';\theta) (\nu \otimes \kappa)(ds\, ds')\\
        &= \iint_{\gS \times \bar{\gS}} f(s, s')h(s, s';\theta) (\nu \otimes \kappa^{b})(ds'\, ds)\\
        &= \iint_{\gS \times \bar{\gS}} f(s, s')u(s';\theta)p_{B}(s', s;\theta) (\nu \otimes \kappa^{b})(ds'\, ds),
    \end{align*}
    where we used $\nu \otimes \kappa = \nu \otimes \kappa^{b}$ in the 3rd inequality. Note that while the equalities between functions are valid $\nu \otimes \kappa$-almost surely over the whole space $\bar{\gS} \times \bar{\gS}$, we only used the equality restricted to $\gS \times \bar{\gS}$. Moreover, since $u$, $p_{F}$, and $p_{B}$ are the densities of the respective measures $\mu$, $P_{F}$, and $P_{B}$ (wrt. the appropriate reference measures), we know that for $B \in \bar{\Sigma} \otimes \bar{\Sigma}$
    \begin{align*}
        (\mu \otimes P_{F})(B) &= \iint_{B} u(s;\theta)p_{F}(s, ds')(\nu\otimes \kappa)(ds\, ds')\\
        (\mu \otimes P_{B})(B) &= \iint_{B} u(s';\theta)p_{B}(s', ds)(\nu \otimes \kappa^{b})(ds'\, ds).
    \end{align*}
    Replacing these measures in the equality above, we obtain
    \begin{equation*}
        \iint_{\gS \times \bar{\gS}} f(s, s')\mu(ds)P_{F}(s, ds') = \iint_{\gS \times \bar{\gS}} f(s, s')\mu(ds')P_{B}(s', ds).
    \end{equation*}
    Since this equality is valid for any bounded measurable function $f$ such that $f(s, s_{0}) = 0$ for all $s \in \gS$, this corresponds to $(\mu, P_{F}, P_{B})$ satisfying the detailed balance conditions (\cref{def:db-conditions}).

Now, for the trajectory balance loss: for a trajectory $(s, s_1, .., s_{n+1}) \in \gS^{n+2}$, we define:
\begin{align*}
    p_F^{\otimes n+1}(s,\overrightarrow{s_{1:n+1}}) &= p_F(s, s_1, \theta)\prod_{t = 1}^{n} p_F(s_t, s_{t+1}, \theta) \\
    p_B^{\otimes n}(\overrightarrow{s_{n:1}},s) &= p_B(s_1,s, \theta) \prod_{t = 1}^{n-1} p_B(s_{t+1}, s_{t}, \theta)
\end{align*}
For any bounded measurable function $f : \bar{\gS}^{n+2} \rightarrow \mathbb{R} $, assuming $L_{TB}^n = 0$ almost surely for every $n \geq 0$: 
\begin{align*}
    & \int_{\gS^{n+2}} Z(\theta)f(s, \overrightarrow{s_{1:n+1}})\indicator_{s=s_0, s_n \neq \bot, s_{n+1} = \bot} p_F^{\otimes n+1}(s,\overrightarrow{s_{1:n+1}})\nu \otimes \kappa^{\otimes n+1}(ds\, \overrightarrow{ds_{1:n+1}})\\
    &= \int_{\{s_0\} \times \gS^{n-1} \times \gX \times \{\bot\}} Z(\theta)f(s, \overrightarrow{s_{1:n+1}}) p_F^{\otimes n+1}(s,\overrightarrow{s_{1:n+1}}) \nu \otimes \kappa^{\otimes n+1}(ds\, \overrightarrow{ds_{1:n+1}})\\
    &= \int_{\{s_0\} \times \gS^{n} \times \{\bot\}} f(s, \overrightarrow{s_{1:n+1}}) r(s_n) p_B^{\otimes n}(\overrightarrow{s_{n:1}},s)  \nu \otimes \kappa^{\otimes n+1}(ds\, \overrightarrow{ds_{1:n+1}}) 
    \\
    &= \int_{\{s_0\}\times \gS^{n}} f(s, \overrightarrow{s_{1:n}}, \bot) r(s_n) \underbrace{\kappa(s_n, {\{\bot\}})}_{=1\ \textrm{(see \cref{eq:kappa-terminal-1})}} p_B^{\otimes n}(\overrightarrow{s_{n:1}},s)  \nu \otimes \kappa^{\otimes n}(ds\, \overrightarrow{ds_{1:n}}) \\
    &=  \int_{\{s_0\} \times \gS^{n}} f(s, \overrightarrow{s_{1:n}}, \bot) r(s_n) p_B^{\otimes n}(\overrightarrow{s_{n:1}},s)  \nu \otimes \kappa^{b, \otimes n}(\overrightarrow{ds_{n:1}}\, ds) \\
    &=  \int_{ \gS^{n+1}} f(s, \overrightarrow{s_{1:n}}, \bot) \indicator_{s = s_0} r(s_n) \nu(ds_n) p_B^{\otimes n}(\overrightarrow{s_{n:1}},s)   \kappa^{b,\otimes n}(s_{n},\overrightarrow{ds_{n-1:1}}\, ds)
\end{align*}

Replacing the measures in the last equality obtained, we recover the TB condition in \cref{def:tb-conditions} with $(Z \nu(\{s_0\}), P_F, P_B)$.

Here, we used $\nu \otimes \kappa^{\otimes n} = \nu \otimes \kappa^{b,\otimes n} $, $\forall n \in \{0,..N\}$. We can show this by simple induction : for $n=0$, it is trivially satisfied . Now suppose it is true for a given $n \leq N-1$, using \cref{eq:backward-kernel}. We have :

\begin{align*}
    \nu \otimes \kappa^{\otimes n+1}(ds \, \overrightarrow{ds_{1:n+1}}) 
    &= \nu \otimes \kappa^{\otimes n}(ds \, \overrightarrow{ds_{1:n}}) \kappa(s_n, ds_{n+1})
    =  \nu \otimes \kappa^{b,\otimes n}(\overrightarrow{ds_{n:1}}\, ds) \kappa(s_n, ds_{n+1})\\
    &= \nu(ds_n) \kappa(s_n, ds_{n+1})\kappa^{b,\otimes n}(s_n, \overrightarrow{ds_{n:1}} \, ds) 
    = \nu(ds_{n+1}) \kappa^b(s_{n+1}, ds_{n})\kappa^{b,\otimes n}(s_n, \overrightarrow{ds_{n:1}} \, ds) \\ 
    &= \nu(ds_{n+1}) \kappa^{b,\otimes n + 1}(s_{n+1}, \overrightarrow{ds_{n:1}} \, ds) 
    = \nu \otimes \kappa^{b,\otimes n+1}(\overrightarrow{ds_{n+1:1}}\, ds)
\end{align*}

Which proves the claim above.

\end{proof}

\section{Proofs of lemmas and propositions in \cref{appendix:transition-kernels}}
\label{appendix:proofs-of-appendices}

\begin{customlemma}{\ref{lemma:absolute-continuity-powers}}
\end{customlemma}

\begin{proof}
We prove the lemma by induction on $n$. The base case ($n=1$) is trivially satisfied. Assuming the property holds for some $n \geq 0$, let $s \in \bar{\gS}$ and $B \in \Sigma$ such that $\kappa^{n+1}(s, B) = 0$.
\begin{equation*}
    P_F^{n+1}(s, B) = \int_{\bar{\gS}} P_F^n(s, ds') P_F(s', B).
\end{equation*}
If $P_F^{n+1}(s, B) > 0$, that would mean there exists an open set $B' \in \gT$ such that $P_F^n(s, B') > 0$ and $P_F(s', B) > 0$ for all $s' \in B'$. From the induction hypothesis, it would follow that $\kappa^n(s, B') > 0$ and $\kappa(s', B)>0$ for all $s' \in B'$, meaning that:
\begin{align*}
    \kappa^{n+1}(s, B) = \int_{\bar{\gS}} \kappa^n(s, ds') \kappa(s', B) \geq \int_{B'} \kappa^n(s, ds') \kappa(s', B) > 0.
\end{align*}
A contradiction ! Hence, $P_F^{n+1}(s, B) = 0$
\end{proof}

\begin{customlemma}{\ref{lemma:terminating-state-dist-PF}}
\end{customlemma}
\begin{proof}
    Starting from the definition of $P_\top^n$:
    \begin{align*}
        \int_{\gX} f(x) P_{\top}^n (dx) = \int_{\bar{\gS}} f(s_n) \indicator_\gX(s_n) P_{\top}^n (dx) 
        = \int_{\bar{\gS}^{n+1}} \indicator_\gX(s_n) f(s_n) P_F^{\otimes{n}}(s_0, ds_1 \dots ds_{n+1})  \indicator_{s_{n+1} = \bot}
    \end{align*}
    Hence, using the recursive definition of $P_F^{\otimes{n}}$ in \cref{eq:n-step-trajectory-measure-recursion}:
    \begin{align*}
        \int_{\gX} f(x) P_{\top}^n (dx)
        &= \int_{\bar{\gS}^{n+1}} \indicator_\gX(s_n) f(s_n) P_F^{\otimes{n-1}}(s_0, ds_1 \dots ds_{n})  P_F(s_n, ds_{n+1})\indicator_{s_{n+1} = \bot}\\
        &=\int_{\bar{\gS}^{n}}  \underbrace{\indicator_\gX(s_n) f(s_n)P_F(s_n, \{\bot\}) }_{g(s_n)}P_F^{\otimes{n-1}}(s_0, ds_1 \dots ds_{n})\\
        &=\int_{\bar{\gS}^n} g(s_n) P_F^{\otimes{n-1}}(s_0, ds_1 \dots ds_{n})\\
        &=\int_{\bar{\gS}} g(s_n) P_F^{n-1}(s_0, ds_n)
        =\int_{\gX}f(x) P_F(x, \{\bot\}P_F^{n-1}(s_0, ds_n),
    \end{align*}
    where we applied \cref{lemma:P_F^n} to the bounded and measurable function $g$.
\end{proof}

\begin{customlemma}{\ref{lemma:existence-of-special-reverse-kernel}}
\end{customlemma}
\begin{proof}
Let $f: \bar{\gS}^2 \ra \R$ be a bounded measurable function.
\begin{align}
    \int_{\bar{\gS}^2} f(s, s') \nu \kappa(ds') \kappa'(s', ds)
    &=\int_{\bar{\gS}} f(s, s_0) \underbrace{\nu \kappa(\{s_0\})}_{=0} \kappa'(s_0, ds) + \int_\gS \int_{\bar{\gS} \setminus \{s_0\}} f(s, s') \nu \kappa(ds') \kappa'(s', ds) \nonumber\\
    & \quad + \int_{\gS \setminus \{s_0\}} f(\bot, s') \nu \kappa(ds') \underbrace{\kappa'(s', \{\bot\})}_{=0} + f(\bot, \bot) \nu \kappa(\{\bot\}) \kappa'(\bot, \{\bot\}) \nonumber\\
    &=\int_\gS \int_{\bar{\gS} \setminus \{s_0\}} f(s, s') \nu \kappa(ds') \kappa^b(s', ds)  + \int_{\bar{\gS}} f(s, s_0) \underbrace{\nu \kappa(\{s_0\})}_{=0} \kappa^b(s_0, ds) \nonumber \\
    & \quad + f(\bot, \bot) \nu \kappa(\{\bot\}) \kappa^b(\bot, \{\bot\}) \label{eq:proof-unique-reverse-kernel}
\end{align}
On the other hand, let $B$ be the largest open set within $\gS$such that $\forall s' \in B, \ \kappa^b(s', \{\bot\}) > 0$. Applying the definition of the reverse kernel \cref{eq:reverse-kernel} to the function $f:(s, s') \mapsto \indicator_{s=\bot} \indicator_{B}(s')$, we get:
\begin{align*}
    \int_{\bar{\gS}} \indicator_{B}(s') \nu(\{\bot\}) \kappa(\bot, s') = \int_{\bar{\gS}} \indicator_{B}(s') \nu \kappa(s') \kappa^b(s', \{\bot\})
\end{align*}
The LHS of the previous equality is $0$, following \cref{eq:absorption}. It follows from the assumption that $\forall s' \in B, \ \kappa^b(s', \{\bot\}) > 0$ that $\nu\kappa(B) = 0$. Hence:
\begin{align*}
    \int_{\gS \setminus \{s_0\}}& f(\bot, s') \nu \kappa(ds') \kappa^b(s', \{\bot\})\\
    &=\int_{\bar{\gS} \setminus \{s_0\}} \indicator_{\gS\setminus B}(s')f(\bot, s') \nu \kappa(ds') \kappa^b(s', \{\bot\})\\
    &\quad + \int_{\bar{\gS} \setminus \{s_0\}} \indicator_{B}(s')f(\bot, s') \nu \kappa(ds') \kappa^b(s', \{\bot\})
\end{align*}
The first summand of the RHS of the last equality is zero by the definition of $B$. The second summand is zero because $\nu \kappa(B) = 0$. Going back to \cref{eq:proof-unique-reverse-kernel}, we obtain:
\begin{align*}
    \int_{\bar{\gS}^2}& f(s, s') \nu \kappa(ds') \kappa'(s', ds) 
    = \int_{\bar{\gS}^2}f(s, s') \nu \kappa(ds') \kappa^b(s', ds)
    = \int_{\bar{\gS}^2}f(s, s') \kappa(ds) \kappa(s, ds')
\end{align*}
\end{proof}

\begin{customproposition}{\ref{prop:existence-of-backward-kernel}}
\end{customproposition}

\begin{proof}
First, note that \cref{eq:absorption,eq:finitely-absorbing} imply that $\forall n \geq N, \ \kappa^n(s_0, -) = \delta_\bot$. This can be shown by a simple induction on $n$, writing for any $B \in \Sigma$:
\begin{align*}
    \kappa^{n+1}(s_0, B) = \int_{\bar{\gS}} \kappa^n(s_0, ds) \kappa(ds, B).
\end{align*}
This entails that \cref{eq:accessibility} could be rewritten as:
\begin{align*}
    \forall B \in \gT \setminus \{\emptyset\}, \ \exists n \in \{0, \dots, N\} \ : \ \kappa^n(s_0, B) > 0.
\end{align*}
Hence, 
\begin{align*}
     \forall B \in \gT \setminus \{\emptyset\}, \ \nu(B) > 0.
\end{align*}
The measure $\nu$ is thus strictly positive.

Note that for any $B$ in $\Sigma_{|\gS}$ such that $s_0 \notin B$.
\begin{align*}
    \nu \kappa(B) &= \int_{\bar{\gS}} \nu(ds) \kappa(s, B) \\
    &= \int_{\bar{\gS}} \sum_{n=0}^N \kappa^n(s_0, ds) \kappa(s, B)
    = \sum_{n=0}^N \int_{\bar{\gS}} \kappa^n(s_0, ds) \kappa(s, B)\\
    &= \sum_{n=0}^N \kappa^{n+1}(s_0, B)
    = \sum_{n=1}^{N+1} \kappa^n(s_0, B).
\end{align*}
Because $B \subseteq \gS$, then $\kappa^{N+1}(s_0, B) = \delta_\bot(B) = 0$. And because $s_0 \notin B$, $\kappa^{0}(s_0, B) = \delta_{s_0}(B) = 0$. From this it follows that:
\begin{equation*}
    \nu \kappa(B) = \nu(B)
\end{equation*}

Then, let $B \in \Sigma_{|\gS} \otimes \Sigma_{|\gS}$ such that $(s_0, s_0) \notin B$. Using the definition of the reverse kernel, we obtain:
\begin{align*}
\nu \otimes \kappa(B) &= (\nu \kappa) \otimes \kappa^r_\nu (B) = \int_{\bar{\gS} \times \bar{\gS}} \indicator_B(s, s') \nu \kappa(ds') \kappa^r_\nu(ds', ds) \\
&=\int_{\gS  \times (\gS\setminus \{s_0\})} \indicator_B(s, s') \underbrace{\nu \kappa(ds')}_{=\nu(ds)} \kappa^r_\nu(ds', ds) + \int_{\gS} \indicator_B(s, s_0) \underbrace{\nu \kappa(s_0)}_{=0} \kappa^r_\nu(s_0, ds)\\
&=\int_{\gS  \times (\gS \setminus \{s_0\})} \indicator_B(s, s') \nu(ds') \kappa^r_\nu(ds', ds) \\
&= \int_{\gS  \times (\gS \setminus \{s_0\})} \indicator_B(s, s') \nu(ds') \kappa^r_\nu(ds', ds) + \int_{\gS \setminus \{s_0\}} \indicator_B(s, s_0) \kappa(s_0) \underbrace{\kappa^r_\nu(s_0, ds)}_{=0} \\
&= \int_{\bar{\gS}  \times \bar{\gS}} \indicator_B(s, s') \nu(ds') \kappa^r_\nu(ds', ds)\\
&= \nu \otimes \kappa^r_\nu(B)
\end{align*}

Finally, if $B \in \Sigma_{| \gS}$:
\begin{align*}
    \nu \otimes \kappa( B \times \{\bot\} ) &= \int_{\gS} \indicator_B(ds) \nu(ds) \kappa(s, \{\bot\}) = \int_{\gS} \indicator_B(ds) \nu(\{\bot\}) \kappa^b(\bot, ds) = \nu \otimes \kappa^b (B \times \{\bot\}) \\
    \nu \otimes \kappa(\{\bot\} \times B ) &= \nu(\{\bot\}) \underbrace{\kappa(\bot, B)}_{=0} = \int_{\gS} \indicator_B(s') \nu(ds') \underbrace{\kappa^b(s', \{\bot\})}_{=0} = \nu \otimes \kappa^b(\{\bot\} \times B )
\end{align*}

\end{proof}